\documentclass[11pt,english]{article}
\usepackage{geometry}
\usepackage[utf8]{inputenc} 
\usepackage[T1]{fontenc}    
\usepackage[hidelinks]{hyperref}       
\hypersetup{
   colorlinks,
   linkcolor={blue},
   citecolor={blue!50!black},
   urlcolor={blue!80!black}
}
\usepackage{url}            
\usepackage{booktabs}       
\usepackage{amsfonts}       
\usepackage{nicefrac}       
\usepackage{microtype}      
\usepackage{xcolor}         
\usepackage{subfig}
\usepackage{dsfont}

\usepackage{amsmath}
\usepackage{enumerate}
\usepackage{empheq}
\usepackage{amsthm}
\usepackage{tabularx}
\usepackage{comment}
\usepackage{comment}
\usepackage{algorithm}
\usepackage{algorithmic}
\usepackage{color}
\usepackage{mathrsfs}
\usepackage{float}
\usepackage{amsmath} 
\usepackage{amssymb}
\usepackage{amsfonts}
\usepackage{amsthm}
\usepackage{stmaryrd}
\usepackage{stackrel}
\usepackage{placeins}
\usepackage{authblk}
\usepackage{subfig}

\usepackage[round]{natbib}
\usepackage{hyperref}
\hypersetup{
colorlinks = true,
urlcolor = blue, 
linkcolor = blue,
citecolor = blue,
}

\newcommand{\Dn}{\mathscr{D}_{n}}
\newcommand{\bX}{\textbf{X}}
\newcommand{\bZ}{\textbf{Z}}
\newcommand{\bTheta}{\boldsymbol{\Theta}}
\newcommand{\bC}{\mathcal{H}}
\newcommand{\bx}{\textbf{x}}

\renewcommand{\P}{\mathds{P}}
\newcommand{\R}{\mathds{R}}
\newcommand{\E}{\mathbb{E}}
\newcommand{\V}{\mathbb{V}}

\newcommand{\VI}{\mathrm{I}^{(j)}}
\newcommand{\VIhat}{\mathrm{I}_n^{(j)}}
\newcommand{\indep}{\perp \!\!\! \perp}

\theoremstyle{plain}
\newtheorem{theorem}{Theorem}

\newtheorem{proposition}{Proposition}
\newtheorem{assumption}{Assumption}
\newtheorem{specification}{Specification}
\newtheorem{definition}{Definition}

\title{\textbf{\LARGE Variable importance for causal forests: breaking down the heterogeneity of treatment effects}}
\author[1]{Clément Bénard}
\affil[1]{Safran Tech, Digital Sciences \& Technologies, 78114 Magny-Les-Hameaux, France}
\author[2]{Julie Josse}
\affil[2]{PreMeDICaL project team, INRIA Sophia-Antipolis, Montpellier, France}

\date{}

\begin{document}

\maketitle

\begin{abstract}
    Causal random forests provide efficient estimates of heterogeneous treatment effects. However, forest algorithms are also well-known for their black-box nature, and therefore, do not characterize how input variables are involved in treatment effect heterogeneity, which is a strong practical limitation. In this article, we develop a new importance variable algorithm for causal forests, to quantify the impact of each input on the heterogeneity of treatment effects.
    The proposed approach is inspired from the drop and relearn principle, widely used for regression problems. Importantly, we show how to handle the forest retrain without a confounding variable. If the confounder is not involved in the treatment effect heterogeneity, the local centering step enforces consistency of the importance measure. Otherwise, when a confounder also impacts heterogeneity, we introduce a corrective term in the retrained causal forest to recover consistency. Additionally, experiments on simulated, semi-synthetic, and real data show the good performance of our importance measure, which outperforms competitors on several test cases. Experiments also show that our approach can be efficiently extended to groups of variables, providing key insights in practice.
\end{abstract}

\sloppy

\section{Introduction}

\subsection{Context and Objectives}

Estimating heterogeneous treatment effects has recently attracted a great deal of interest in the machine learning community, particularly for medical applications \citep{obermeyer2016predicting} and in the social sciences. 
Over the past few years, numerous efficient algorithms have been developed to estimate such effects, including double robust methods \citep{kennedy2020optimal}, R-learners \citep{nie2021quasi}, X-learners \citep{kunzel2019metalearners}, causal forests \citep{wager2018estimation, athey2019generalized}, the lasso \citep{imai2013lasso}, BART \citep{hill2011bayesian}, or neural networks \citep{shalit2017estimating}.  
However, most of these methods remain black boxes, and it is therefore difficult to grasp how input variables impact treatment effects.  This understanding is crucial for optimizing treatment policies, for instance.
 While the accuracy of treatment effect estimates has significantly improved recently, little effort has been dedicated to improve their interpretability, and  quantifying the impact of variables involved in treatment effect heterogeneity. In this regard, we can mention the importance measure of the causal forest package
\texttt{grf} \citep{grf}, the double robust approach of \citep{hines2022variable}, and the algorithm from \citep{boileau2022flexible} for high dimensional linear cases. Besides, let us also mention policy learning, which aims at selecting relevant individuals to treat \citep{zhao2012estimating, swaminathan2015batch, kitagawa2018should, athey2021policy}. However, these policy procedures are also black boxes, which limits their practical use.
The main purpose of this article is to introduce a variable importance measure for heterogeneous treatment effects, improving over the existing algorithms, to better identify the sources of heterogeneity. We focus on causal random forests, defined as a specific case of generalized forests \citep{athey2019generalized}, and well-known to be one of most accurate algorithm to estimate heterogeneous treatment effects.

\paragraph{Contributions.}
Our main contribution is thus the introduction of a variable importance algorithm for causal random forests, following the drop and retrain principle, which is well-established for regression problems \citep{lei2018distribution, williamson2020unified, hooker2021unrestricted, benard2022mean}. The main idea is to retrain the learning algorithm without a given input variable, and measure the drop of accuracy to get its importance. In particular, such approach ensures that irrelevant variables get a null importance asymptotically.
In the context of causal inference, the main obstacle is to retrain the causal forest without a confounding variable, since the unconfoundedness assumption can be violated, leading to inconsistent forest estimates and biased importance values, as explained in Section \ref{sec:vimp}. However, we will see that the local centering of the outcome and treatment assignment leads to consistent estimates, provided that the removed variable is not involved in the treatment effect heterogeneity. Otherwise, to handle a confounder involved in heterogeneity, we introduce a corrective term in the retrained causal forest. Overall, we will show in Section \ref{sec:theory}, that our proposed variable importance algorithm is consistent, under standard assumptions in the literature about the theoretical analysis of random forests. Next, in Section \ref{sec:xp}, we run several batches of experiments on simulated, semi-synthetic, and real data to show the good performance of the introduced method compared to the existing competitors. Additionally, we take advantage of the experimental section to illustrate that the extension of our approach to group of variables is straightforward and provides powerful insights in practice.
The remaining of this first section is dedicated to the mathematical formalization of the problem.

\subsection{Definitions}
To define heterogeneous treatment effects, we first introduce a standard causal setting with an input vector $\bX = (X^{(1)}, \hdots, X^{(p)}) \in \R^p$ with $p \in \mathbb{N}^{\star}$, the binary treatment assignment $W \in \{0, 1\}$, the potential outcome $Y(1) \in \mathbb{R}$ for the subject receiving the treatment, and the potential outcome without treatment $Y(0)  \in \mathbb{R}$. We denote by $\smash{\bX^{(\bC)}}$  the subvector with only the components in $\bC \subset \{1,\hdots,p\}$, and $\bX^{(-j)}$ the vector $\bX$ with the $j$-th component removed. The observed outcome is given by $Y = W Y(1) + (1 - W) Y(0)$, which is known as the SUTVA assumption in the literature. More precisely, the potential outcomes are defined by
\begin{align*}
    &Y(0) = \mu(\bX) + \varepsilon(0), \\
    &Y(1) = \mu(\bX) + \tau(\bX^{(\bC)}) + \varepsilon(1),    
\end{align*}
where $\mu(\bX)$ is a baseline function, $\tau(\bX^{(\bC)})$ is the conditional average treatment effect (CATE) only depending on variables in $\bC  \subset \{1,\hdots,p\}$, and $\varepsilon(0), \varepsilon(1)$ are some noise variables satisfying $\E[\varepsilon(0) \mid \bX] = \E[\varepsilon(1) \mid \bX] = 0$. 
Notice that the CATE is also defined as the mean difference between potential outcomes, conditional on $\bX$, i.e., $\E[Y(1) - Y(0) \mid \bX] = \tau(\bX^{(\bC)})$, by construction. Overall, the observed outcome $Y$ also writes
\begin{align*}
    Y = \mu(\bX) + \tau(\bX^{(\bC)}) \times W + \varepsilon(W).
\end{align*}

The cornerstone of causal treatment effect identifiability  is the assumption of unconfoundedness given below, which states that all confounding variables are observed in the data. By definition, the responses $Y(0)$, $Y(1)$, and the treatment assignment $W$ simultaneously depend on the confounding variables. If all confounding variables are observed, then the responses and the treatment assignment are independent conditional on the inputs. Consequently, the treatment effect is identifiable, as stated in the following proposition---all proofs of propositions and theorems stated throughout the article are gathered in Appendix \ref{App:proofs}. Notice that Assumption \ref{A:unconfound} below enforces that the input vector $\bX$ contains all confounding variables, but $\bX$ may also contain non-confounding variables. Consequently, $\bX^{(\bC)}$ can also be a mix of confounding and non-confounding variables, or contain only variables of one type. Ideally, all variables impacting the treatment effect heterogeneity should be involved in the analysis, even if they are not confounding variables, to better estimate and interpret the treatment effect.
\begin{assumption} \label{A:unconfound}
    Potential outcomes are independent of the treatment assignment conditional on the observed input variables, i.e., $Y(0), Y(1) \indep W \mid \bX$.
\end{assumption}
\begin{proposition} \label{prop_identifiability}
    If the unconfoundedness Assumption \ref{A:unconfound} is satisfied, then we have
    \begin{align*}
        \tau(\bX^{(\bC)}) = \E[Y \mid \bX, W = 1] - \E[Y \mid \bX, W = 0].
    \end{align*}
\end{proposition}

Note that we define above the treatment effect as the expected difference between potential outcomes, conditioned on input variables. However, the heterogeneity properties strongly depend on how we define the treatment effect \citep{vanderweele2007four, rothman2012epidemiology, colnet2023risk}. The ratio between the means of potential outcomes may also define a treatment effect, leading to potential heterogeneity while our original outcome difference remains constant.  A thorough discussion of this topic is out of scope of this article, and we take the difference of potential outcomes as treatment effect, the widely used metric for many applications \citep{vanderweele2007four}. We refer to \citet{colnet2023risk} for a comparison of treatment effect measures.

\citet{vanderweele2007four} defined treatment effect heterogeneity as follows.
\begin{definition}[\citet{vanderweele2007four}] \label{def:heterogeneous}
    The treatment effect $\tau$ is said to be heterogeneous with respect to $\bX$ if it exists $\bx, \bx' \in \R^p$ such that $\tau(\bx^{(\bC)}) \neq \tau(\bx'^{(\bC)})$.
\end{definition}
We strengthen this definition in two directions, formalized in Definition \ref{A:tau} below. First, we require $\tau$ to be heterogeneous with respect to each variable in $\bC$, to enforce $\bC$ to be the subset of variables impacting treatment effect heterogeneity. Secondly, notice that Definition \ref{def:heterogeneous} can be satisfied while having an homogeneous treatment effect in probability, i.e., $\P(\tau(\bX^{(\bC)}) = \tau(\bX'^{(\bC)})) = 1$, with $\bX'^{(\bC)}$ an independent copy of $\bX^{(\bC)}$.  In such cases, heterogeneity is not detectable from a data sample, and has a negligible impact in practice. Therefore, we enforce $\tau$ to take distinct values with respect to all variables in $\bC$ on sets of non-null Lebesgue measure.
\begin{definition} \label{Def:tau}
     The treatment effect $\tau$ is said to be heterogeneous with respect to all variables in $\bC$, if for all $j \in \bC$, it exists $\mathcal{X}_{p-1} \subset \R^{p-1}$ and $\mathcal{X}_1, \mathcal{X}'_1  \subset \R$, such that for all $\bx^{(-j)} \in \mathcal{X}_{p-1}$, $x^{(j)}  \in \mathcal{X}_1$, $x'^{(j)} \in \mathcal{X}'_1$, we have
     \begin{align*}
         \tau(\bx^{(\bC)}) \neq \tau(\bx'^{(\bC)}),
     \end{align*}
     with $\bx'^{(-j)} = \bx^{(-j)}$, and $\mathcal{X}_{p-1}$, $\mathcal{X}_1$, and $\mathcal{X}'_1$ have a non-null Lebesgue measure.
\end{definition}
In the sequel, we assume that the treatment effect $\tau$ is heterogeneous in the sense of Definition \ref{A:tau}, and that $\bX$ admits a strictly positive density, to enforce heterogeneity with a positive probability, as stated in the proposition below. Our objective is to quantify the influence of the input variables $\bX$ on the treatment heterogeneity  using an available sample $\Dn = \{(\bX_i, Y_i, W_i)\}_{i=1}^n$, made of $n \in \mathbb{N}^{\star}$ independent and identically distributed (iid) observations.
\begin{assumption} \label{A:tau}
    The treatment effect $\tau$ is heterogeneous according to Definition \ref{A:tau}, and $\bX$ admits a strictly positive density.
\end{assumption}
\begin{proposition} \label{prop:heterogeneous_prob}
    If Assumption \ref{A:tau} is satisfied, and $\bX'$ is an independent copy of $\bX$, then 
    \begin{align*}
        \P(\tau(\bX^{(\bC)}) \neq \tau(\bX'^{(\bC)})) > 0.
    \end{align*}
\end{proposition}

\section{Variable Importance for Heterogeneous Treatment Effects} \label{sec:vimp}

\subsection{Theoretical Definition}

To propose a variable importance measure, we build on \citet{sobol1993sensitivity} and \citet{williamson2020unified}, which define variable importance in the case of regression as the proportion of output explained variance lost when a given input variable is removed. \citet{hines2022variable} extend this idea to treatment effects, and introduce the theoretical importance measure $\VI$ of $X^{(j)}$, defined by
\begin{align} \label{eq:th_vimp}
    \VI = \frac{\V[\tau(\bX^{(\bC)})] - \V[\E[\tau(\bX^{(\bC)})|\bX^{(-j)}]]}{\V[\tau(\bX^{(\bC)})]} = \frac{\E[(\tau(\bX^{(\bC)}) - \E[\tau(\bX^{(\bC)})|\bX^{(-j)}])^2]}{\V[\tau(\bX^{(\bC)})]},
\end{align}
which is well-defined under Assumption \ref{A:tau}, since $\V[\tau(\bX^{(\bC)})] > 0$. Otherwise, when $\V[\tau(\bX^{(\bC)})] = 0$, the treatment is homogeneous, i.e. constant with respect to all input variables, and does not satisfy Definition \ref{Def:tau}.
This importance measure gives the proportion of treatment effect variance lost when a given input variable is removed. Additionally, the following proposition shows that $\VI$ properly identifies variables in $\bC$, which have an impact on treatment heterogeneity, where the proof in Appendix \ref{App:proofs} is a consequence of Assumption \ref{A:tau}.
\begin{proposition} \label{prop_VI}
    Let Assumption \ref{A:tau} be satisfied. If $j \notin \bC$, then we have $\VI = 0$.
    Otherwise, if $j \in \bC$, we have $0 < \VI \leq 1$. 
\end{proposition}
Note that by definition of $\VI$, a variable strongly correlated to the other inputs, has a low importance value. This is due to the fact that, owing to this strong dependence, there is minimal loss of information regarding the treatment effect heterogeneity when such a variable is removed. As suggested by both \citet{williamson2020unified} and \citet{hines2022variable}, one possible approach involves extending the importance measure to a group of variables, where strongly dependent variables are grouped together. For the sake of clarity, we focus on the case of a single variable in the following sections. However, extending this approach to groups of variables is straightforward, and we will present such examples in the experimental section.

More importantly, \citet{hines2022variable} highlight that a key problem to estimate the above quantity $\VI$, is that the unconfoundedness Assumption~\ref{A:unconfound} does not imply unconfoundedness for the reduce set of input variables $\bX^{(-j)}$, i.e., we may have $Y(0), Y(1) \not \indep W \mid \bX^{(-j)}$. \citet{hines2022variable} overcome this issue using double robust approaches \citep{kennedy2020optimal, nie2021quasi} to estimate $\tau$ with all input variables in a first step, and then regress the obtained treatment effect on $\bX^{(-j)}$ to estimate $\E[\tau(\bX^{(\bC)})|\bX^{(-j)}]$. Actually, the generalized random forest framework from \citet{athey2019generalized} enables to get closer to the original proposal of \citet{williamson2020unified} by retraining the causal forest without variable $X^{(j)}$ and still get consistent estimates of $\E[\tau(\bX^{(\bC)})|\bX^{(-j)}]$, as we will see. Therefore, we focus on causal forests \citep{wager2018estimation, athey2019generalized}, one of the state-of-the-art algorithm to estimate heterogeneous treatment effects, to propose efficient estimates of $\VI$.

\subsection{Causal Random Forests}

Generalized random forests \citep{athey2019generalized} are a generic framework to build efficient estimates of quantities defined as solutions of local moment equations. As opposed to original Breiman's forests, generalized forests are not the average of tree outputs. Instead, trees are aggregated to generate weights for each observation of the training data, used in a second step to build a weighted estimate of the target quantity. Causal forests are a specific case of generalized forest, where the following local moment equation identifies the treatment effect under the unconfoundedness Assumption \ref{A:unconfound},
\begin{align} \label{eq:local_moment_full}
    \tau(\bX^{(\bC)}) \times \V[W \mid \bX] - \mathrm{Cov}[W, Y \mid \bX] = 0.
\end{align}
The local moment equation (\ref{eq:local_moment_full}) is thus used to define the causal forest estimate $\tau_{M,n}(\bx)$ at a new query point $\bx$, built from the data $\Dn$ with $M \in \mathbb{N}^{\star}$ trees, and formally defined in \citet[Section $6.1$]{athey2019generalized} by
\begin{align} \label{eq:causal_forest}
    \tau_{M,n}(\bx) = \frac{\sum_{i=1}^n \alpha_i(\bx) W_i Y_i - \overline{W}_{\alpha} \overline{Y}_{\alpha}}{\sum_{i=1}^n \alpha_i(\bx) (W_i - \overline{W}_{\alpha})^2},
\end{align}
where $\overline{Y}_{\alpha} = \sum_{i=1}^n \alpha_i(\bx) Y_i$, $\overline{W}_{\alpha} = \sum_{i=1}^n \alpha_i(\bx) W_i$, and the weights $\alpha_i(\bx)$ are generated by the forest to quantify the frequency of $\bx$ and the training observation $\bX_i$ both falling in the same terminal leaves of trees. Notice that the $\ell$-th tree of the forest is randomized by $\Theta_{\ell}$, which defines the resampling of the data prior to the tree growing, as well as the random variable selection at each node for the split optimization. We write the causal forest estimate $\tau_{M,n}(\bx, \bTheta_M)$ when it improves clarity, where $\bTheta_M = (\Theta_{1}, \hdots, \Theta_{M}$). Besides, notice that the local moment equation (\ref{eq:local_moment_full}) is also used to define an efficient splitting criterion of the tree nodes.

Finally, the causal forest algorithm first performs a local centering step in practice, by regressing $Y$ and $W$ on $\bX$ using regression forests, fit with $\Dn$. The obtained out-of-bag forest estimates of $m(\bX_i) = \E[Y_i \mid \bX_i]$ and $\pi(\bX_i) = \E[W_i \mid \bX_i]$ are denoted by $\smash{\hat{m}_n(\bX_i)}$ and $\smash{\hat{\pi}_n(\bX_i)}$. 
Then, these quantities are subtracted to get the centered outcome $\smash{\tilde{Y}_i = Y_i - \hat{m}_n(\bX_i)}$, and  centered treatment $\smash{\tilde{W}_i = W_i - \hat{\pi}_n(\bX_i)}$, used to fit the causal forest $\tau_{M,n}(\bx)$.

\subsection{Variable Importance Algorithm}

We take advantage of causal forests to build an estimate of our variable importance measure $\smash{\VI}$, defined in equation (\ref{eq:th_vimp}).
The forest estimate $\smash{\tau_{M,n}(\bx)}$, described in the previous subsection, provides a plug-in estimate for the first term $\smash{\tau(\bX^{(\bC)})}$  of $\smash{\VI}$.
Next, we need to estimate the second term $\smash{\E[\tau(\bX^{(\bC)})|\bX^{(-j)}]}$ involved in $\smash{\VI}$, and then, a Monte-Carlo method will provide an efficient algorithm for our importance measure.
Hence, a natural approach is to drop the $j$-th variable and retrain the forest to estimate $\smash{\E[\tau(\bX^{(\bC)})|\bX^{(-j)}]}$.  As we deepen below and summarize in Algorithm \ref{algo_VI}, a critical feature of this procedure is that all input variables  are used in the local centering of $Y_i$ and $W_i$, before the $j$-th variable is dropped to build $\smash{\tau_{M,n}^{(-j)}(\bx)}$.
Therefore, the causal forest is retrain using the observations $\smash{\{(\bX_i^{(-j)}, \tilde{Y}_i, \tilde{W}_i)\}_{i=1}^n}$ to generate new weights $\smash{\alpha'(\bx^{(-j)})}$ and build $\smash{\tau_{M,n}^{(-j)}(\bx)}$ through equation (\ref{eq:causal_forest}). 

\paragraph{Identifiability of treatment effect.} 
When a variable $X^{(j)}$ is removed from the input variables, the moment equation (\ref{eq:local_moment_full}) does not necessarely hold anymore, since unconfoundedness Assumption (\ref{A:unconfound}) may be violated with a reduced set of inputs. However, an important feature of causal forests is the preliminary step of local centering of the observed outcome and treatment assignment, explained above. The following proposition shows that the treatment effect is well identified by the local moment equation of causal forests including only variables in $\bC$, provided that the data is centered with all inputs. We recall that $m(\bX) = \E[Y \mid \bX]$ and $\pi(\bX) = \E[W \mid \bX]$.
\begin{proposition} \label{prop:local_moment_C}
    If Assumption \ref{A:unconfound} is satisfied, we have
    \begin{align*}
        \quad \tau(\bX^{(\bC)}) \times \V[W - \pi(\bX) \mid \bX^{(\bC)}] - \mathrm{Cov}[W - \pi(\bX), Y - m(\bX) \mid \bX^{(\bC)}] = 0,
    \end{align*} 
    which is the local moment equation defining causal forests, with input variables $\bX^{(\bC)}$, centered outcome $Y - m(\bX)$, and centered treatment assignment $W - \pi(\bX)$.
\end{proposition}
On the other hand, removing an influential variable $j \in \bC$ to learn a causal forest is more delicate. Indeed, a local moment equation to identify the mean CATE over $X^{(j)}$ exists if the treatment effect is uncorrelated to the squared centered treatment assignment.
\begin{proposition} \label{prop:local_moment_j}
    If Assumption \ref{A:unconfound} is satisfied, then we have for $j \in \bC$
    \begin{align*}
        \quad \E[\tau(\bX^{(\bC)}) \mid \bX^{(-j)}] \times \V[W - \pi(\bX) \mid &\bX^{(-j)}] - \mathrm{Cov}[W - \pi(\bX), Y - m(\bX) \mid \bX^{(-j)}] \\ &+ \mathrm{Cov}[\tau(\bX^{(\bC)}), \pi(\bX)(1 - \pi(\bX)) \mid \bX^{(-j)}] = 0.
    \end{align*} 
    Then, for a query point $\bx^{(-j)} \in [0,1]^{p-1}$, if $\mathrm{Cov}[\tau(\bX^{(\bC)}), \pi(\bX)(1 - \pi(\bX)) \mid \bX^{(-j)} = \bx^{(-j)}]= 0$, $\E[\tau(\bX^{(\bC)}) \mid \bX^{(-j)} = \bx^{(-j)}]$ is identified by the original local moment equation of causal forests, with $\bX^{(-j)}$ as input variables, centered outcome $Y - m(\bX)$, and centered treatment assignment $W - \pi(\bX)$. 
\end{proposition}
\citet[Footnote $5$, page $42$]{athey2019estimating} conduct an empirical analysis using causal forests, and state in a footnote, that local centering ``eliminates confounding effects. Thus, we do not need to give the causal forest all features $X^{(j)}$ that may be confounders. Rather, we can focus on features that we
believe may be treatment modifiers''. However, Propositions \ref{prop:local_moment_C} and  \ref{prop:local_moment_j} show that this statement must be completed. Indeed, Proposition \ref{prop:local_moment_C} states that confounders not involved in the heterogeneity of the treatment effect, i.e. confounders that do no belong to $\bC$, may be dropped without hurting the identifiability of $\tau$, thanks the the local centering step. On the other hand, Proposition \ref{prop:local_moment_j} shows that this is clearly not the case for confounders involved in heterogeneity, as the treatment effect is not properly identified by the local moment equation of causal forests, even with local centering. To overcome this problem, we introduce a corrective term in the retrained forest.

\paragraph{Corrected causal forests.}
The additional covariance term in Proposition \ref{prop:local_moment_j} can be estimated using the original causal forest fit with all inputs. Therefore, we propose the corrected causal forest estimate when removing a confounding variable $X^{(j)}$ with $j \in \bC$. 
Recall that the weights $\alpha'(\bx^{(-j)})$ are generated by the causal forest using centered data and dropping variable $X^{(j)}$, to define $\smash{\tau_{M,n}^{(-j)}(\bx)}$.
We define the corrected causal forest estimate $\smash{\theta_{M,n}^{(-j)}(\bx)}$ as
\begin{align} \label{eq:correction}
    \theta_{M,n}^{(-j)}(\bx) = \tau_{M,n}^{(-j)}(\bx) - \frac{\sum_{i=1}^n \alpha'_i(\bx^{(-j)}) \tilde{W}_i^{2} \tau_{M,n}(\bX_i) - \overline{W^{2}_{\alpha'}} \overline{\tau}_{\alpha'}}{\overline{W^{2}_{\alpha'}} - (\overline{W}_{\alpha'})^2},
\end{align}
where $\overline{W^{2}_{\alpha'}} = \sum_{i=1}^n \alpha'_i(\bx^{(-j)}) \tilde{W}_i^{ 2}$, $\overline{W}_{\alpha'} = \sum_{i=1}^n \alpha'_i(\bx^{(-j)}) \tilde{W}_i$, and the mean treatment effect is $\overline{\tau}_{\alpha'} = \sum_{i=1}^n \alpha'_i(\bx^{(-j)}) \tau_{M,n}(\bX_i)$. With such correction, the causal forest retrained without a confounding variable is consistent, as we will show in Section \ref{sec:theory}.
Note however that, in practice, the correction term can be small, as  demonstrated in the experimental Section \ref{sec:xp}. 

\paragraph{Variable importance estimate.}
Using $\Dn' = \{(\bX'_i, Y'_i, W'_i)\}_{i=1}^n$ an independent copy of $\Dn$, we define
\begin{align} \label{eq:VIhat}
    \VIhat = \frac{\sum_{i = 1}^n \big[\tau_{M,n}(\bX'_i) - \theta^{(-j)}_{M,n}(\bX'_i)\big]^2}{\sum_{i = 1}^n \big[\tau_{M,n}(\bX'_i) - \overline{\tau_{M,n}} \big]^2} - \mathrm{I}_n^{(0)},
\end{align}
where $\overline{\tau_{M,n}} = \sum_{i = 1}^n \tau_{M,n}(\bX'_i)/n$, and $\mathrm{I}_n^{(0)}$ is the mean squared difference between the initial forest predictions and the predictions of the corrected forest $\smash{\theta^{(0)}_{M,n}(\bX'_i, \Theta'_M)}$, retrained with still all the inputs variables involved but a new randomization $\Theta'_{M}$, i.e.,
\begin{align*}
    \mathrm{I}_n^{(0)} = \frac{\sum_{i = 1}^n \big[\tau_{M,n}(\bX'_i, \Theta_M) - \theta^{(0)}_{M,n}(\bX'_i, \Theta'_M)\big]^2}{\sum_{i = 1}^n \big[\tau_{M,n}(\bX'_i) - \overline{\tau_{M,n}} \big]^2}.
\end{align*}
In fact, $\mathrm{I}_n^{(0)}$ partially removes the bias of the first term of $\VIhat$, due to the randomization of the forest training, and vanishes as the sample size increases if the causal forest converges.
Notice that the above definition is formalized with $\Dn'$ for the sake of clarity, but that such additional data is usually not available in practice. Instead, out-of-bag causal forest estimates are rather used to define $\smash{\VIhat}$, as summarized in Algorithm \ref{algo_VI} below.


\begin{algorithm}
\caption{Variable importance algorithm for causal forests}
\label{algo_VI}
\begin{algorithmic}[1]
\REQUIRE A dataset $\Dn = \{(\bX_i, Y_i, W_i)\}_{i=1}^n$ containing all confounding variables.
\STATE Perform local centering of outputs $Y_i$ and treatment assignments $W_i$ to get the centered dataset $\smash{\{(\bX_i, \tilde{Y}_i, \tilde{W}_i)\}_{i=1}^n}$, using regression forests and out-of-bag estimates.
\STATE Train a causal forest with the centered data $\smash{\{(\bX, \tilde{Y}_i, \tilde{W}_i)\}_{i=1}^n}$ containing all variables.
\FOR{$j \in \{1,\hdots,p\}$}
    \STATE Train a corrected causal forest with the centered data $\{(\bX^{(-j)}, \tilde{Y}_i, \tilde{W}_i)\}_{i=1}^n$, where the $j$-th variable is removed.
    \STATE Compute $\VIhat$ according to equation (\ref{eq:VIhat}) and using the initial forest and the retrained forest of the previous step.
\ENDFOR
\RETURN $\big\{\VIhat\big\}_{j=1}^p$
\end{algorithmic}
\end{algorithm}

\section{Theoretical Properties} \label{sec:theory}

Propositions \ref{prop:local_moment_C} and \ref{prop:local_moment_j} are the cornerstones of the consistency of our variable importance algorithm. This result relies on the asymptotic analysis of \citet{athey2019generalized}, which states the consistency of causal forests in Theorem \ref{thm:cf_consistency}. Several mild assumptions are required, mainly about the input distribution, the regularity of the involved functions, and the forest growing.
Then, the core of our mathematical analysis is the extension to the case of a causal forest fit without a given input variable. When the removed input is a confounding variable, consistency is obtained thanks to the corrective term introduced in equation (\ref{eq:correction}) of the previous section. Then, the convergence of our variable importance algorithm follows using a standard asymptotic analysis.
We first formalize the required assumptions and specifications on the tree growing from \citet{athey2019generalized}, that are frequently used in the theoretical analysis of random forests \citep{meinshausen2006quantile, scornet2015consistency, wager2018estimation}.
\begin{assumption} \label{A:X_density}
    The input $\bX$ takes value in $[0,1]^p$, and admits a density bounded from above and below by strictly positive constants.
\end{assumption}
\begin{assumption} \label{A:lipschitz}
    The functions $\pi$, $m$, and $\tau$ are Lipschitz, $0 < \pi(\bx) < 1$ for $\bx \in [0,1]^p$, and $\mu$ and $\tau$ are bounded.
\end{assumption}
\begin{specification} \label{Spec:forests}
    Tree splits are constrained to put at least a fraction $\gamma > 0$ of the parent node observations in each child node. The probability to split on each input variable at every tree node is greater than $\delta > 0$. The forest is honest, and built via subsampling with subsample size $a_n$, satisfying $a_n/n \to 0$ and $a_n \to \infty$. 
\end{specification}
The first part of Specification \ref{Spec:forests} is originally introduced by \citet{meinshausen2006quantile}. The idea is to enforce the diameter of each cell of the trees to vanish as the sample size increases, by adding a constraint on the minimum size of children nodes, and slightly increasing the randomization of the variable selection for the split at each node. Then, vanishing cell diameters combined to Lipschitz functions lead to the forest convergence. Additionally, honesty is a key property of the tree growing, extensively discussed in \citet{wager2018estimation}, where half of the data is used to optimize the splits, and the other half to estimate the cell outputs. With these assumptions satisfied, we state below the causal forest consistency proved in \citet{athey2019generalized}. Notice that the original proof is conducted for generalized forests, for any local moment equations satisfying regularity assumptions, automatically fulfilled for the moment equation (\ref{eq:local_moment_full}) involved in our analysis. In Appendix \ref{App:proofs}, we give a specific proof of Theorem \ref{thm:cf_consistency} in the case of causal forests. We built on this proof to further extend the consistency result when a confounding variable is removed.
\begin{theorem}[Theorem $3$ from \citet{athey2019generalized}] \label{thm:cf_consistency}
    If Assumptions \ref{A:unconfound}-\ref{A:lipschitz} and Specification \ref{Spec:forests} are satisfied, and the causal forest $\tau_{M,n}(\bx)$ is built with $\Dn$ without local centering, then we have for $\bx \in [0,1]^p$,
    \begin{align*}
        \tau_{M,n}(\bx) \overset{p}{\longrightarrow} \tau(\bx^{(\bC)}).
    \end{align*}
\end{theorem}

Next, we need a slight simplification of our variable importance algorithm to alleviate the mathematical analysis. We assume that a centered dataset $\Dn^{\star} = \{(\bX_i, W_i^{\star}, Y_i^{\star})\}$ is directly available, where $W_i^{\star} = W_i - \pi(\bX_i)$ and $Y_i^{\star} = Y_i - m(\bX_i)$. A causal forest grown with this dataset where a given input variable $j \in \{1,\hdots,p\} \setminus \bC$ is dropped, consistently estimates the treatment effect as stated below. Consistency also holds for variables $j \in \bC$ in specific cases, whereas in the general case, the corrected term introduced in equation (\ref{eq:correction}) is required. 
Theorem \ref{thm:cf_consistency_centered} states the consistency of causal forests when an input variable is removed.
\begin{theorem} \label{thm:cf_consistency_centered}
    If Assumptions \ref{A:unconfound}-\ref{A:lipschitz} and Specification \ref{Spec:forests} are satisfied, and the causal forest $\tau_{M,n}^{(-j)}(\bx)$ is fit with the centered data $\smash{\Dn^{\star (-j)}}$ without the $j$-th variable,
    
    (i) for $j \in \{1,\hdots,p\} \setminus \bC$ and $\bx \in [0,1]^p$, we have
    \begin{align*}
        \tau_{M,n}^{(-j)}(\bx) \overset{p}{\longrightarrow} \tau(\bx^{(\bC)}),
    \end{align*}
    
    (ii) for $j \in \bC$  and $\bx \in [0,1]^p$, if $\mathrm{Cov}[\tau(\bX^{(\bC)}), \pi(\bX)(1 - \pi(\bX)) \mid \bX^{(-j)} = \bx^{(-j)}]= 0$, we have
    \begin{align*}
        \tau_{M,n}^{(-j)}(\bx) \overset{p}{\longrightarrow} \E[\tau(\bX^{(\bC)}) \mid \bX^{(-j)} = \bx^{(-j)}].
    \end{align*}
\end{theorem}
Theorem \ref{thm:cf_consistency_centered} is a direct consequence of Propositions \ref{prop:local_moment_C} and \ref{prop:local_moment_j} combined with Theorem \ref{thm:cf_consistency}. Indeed, provided that the outcome and treatment assignment are centered, if the removed variable $j$ is not involved in the treatment heterogeneity, i.e. $j \notin \bC$, consistency holds. On the other hand, if $j \in \bC$, we need an additional assumption that $\tau(\bX^{(\bC)})$ and $\pi(\bX)(1 - \pi(\bX))$ are not correlated conditional on $\bX^{(-j)} = \bx^{(-j)}$, where $\bx^{(-j)}$ is the new query point. Otherwise, consistency is obtained with a corrective term defined in equation (\ref{eq:correction}), as we will see. However, we need an additional small modification of causal forests to enforce the generated estimates to be bounded, and to limit the number of observations in each terminal leave of trees, as stated in the specification below. 
Notice that such modifications are quite mild. Indeed, the true treatment effect is bounded by assumption. For the second part, the number of observations in each terminal leave may not be bounded in specific cases, because of honest tree growing. Nevertheless, it is still possible to comply with this specification, by randomly splitting cells that exceed the number of observation threshold.
\begin{specification} \label{Spec:truncated}
    The causal forest estimates are truncated from below and above by $-K$ and $K$, where $K \in \mathbb{R}$ is an arbitrarily large constant. The number of observations in each terminal leave of trees is smaller than a threshold $t_0 \in \mathbb{N}^{\star}$.
\end{specification}
\begin{theorem} \label{thm:cf_consistency_corrected}
    Let the initial causal forest $\tau_{M,n}(\bx)$ fit with the centered data $\Dn^{\star}$, and the corrected causal forest $\smash{\theta_{M,n}^{(-j)}(\bx)}$ fit using $\tau_{M,n}(\bx)$ and  $\smash{\Dn^{\star (-j)}}$, an independent copy of the centered data with the $j$-th variable dropped.
    If Assumptions \ref{A:unconfound}-\ref{A:lipschitz}, and Specifications \ref{Spec:forests} and \ref{Spec:truncated} are satisfied, then for $j \in \{1,\hdots,p\}$ and $\bx \in [0,1]^p$, we have
    \begin{align*}
        \theta_{M,n}^{(-j)}(\bx) \overset{p}{\longrightarrow} \E[\tau(\bX^{(\bC)}) \mid \bX^{(-j)} = \bx^{(-j)}].
    \end{align*}
\end{theorem}
Since Theorems \ref{thm:cf_consistency} and \ref{thm:cf_consistency_corrected} give the consistency of causal forests respectively fit with all input variables, and when a given variable is removed, we can deduce the consistency of our variable importance algorithm from standard asymptotic arguments.
\begin{theorem} \label{thm:vimp_consistency}
    Under the same assumptions than Theorem \ref{thm:cf_consistency_corrected}, we have for all $j \in \{1,\hdots,p\}$
    \begin{align*}
        \VIhat \overset{p}{\longrightarrow} \VI.
    \end{align*}
\end{theorem}
Theorem \ref{thm:vimp_consistency} states that the introduced variable importance algorithm gets arbitrarily close to the true theoretical value, provided that the sample size is large enough. Combining this result with Proposition \ref{prop_VI}, we get that, for $j \notin \bC$, $\VIhat \overset{p}{\longrightarrow} 0$, which means that the variables not involved in the treatment heterogeneity by construction get a null importance. 
Finally, we conclude our theoretical analysis with a focus on the corrective term of the retrained causal forests. In particular, we quantify the positive asymptotic bias introduced in the importance measure without this correction. We thus denote by $\smash{\mathcal{I}_n^{(j)}}$ the estimated importance measure following the same procedure as for $\smash{\VIhat}$, except that the corrected forest $\smash{\theta_{M,n}^{(-j)}(\bx)}$ is replaced by the raw retrained forest $\smash{\tau_{M,n}^{(-j)}(\bx)}$.
\begin{theorem} \label{thm:vimp_consistency_bias}
    Under the same assumptions than Theorem \ref{thm:cf_consistency_corrected}, with $\mathcal{I}_n^{(j)}$ the importance measure estimated without the corrective term in the causal forests, we have for all $j \in \bC$,
    \begin{align*}
        \mathcal{I}_n^{(j)} \overset{p}{\longrightarrow} \VI 
        + \frac{1}{\V[\tau(\bX^{(\bC)})]} \E\Big[ \frac{\mathrm{Cov}[\tau(\bX^{(\bC)}), \pi(\bX)(1 - \pi(\bX)) \mid \bX^{(-j)}]^2}{\E[\pi(\bX)(1 - \pi(\bX)) \mid \bX^{(-j)}]^2} \Big].
    \end{align*}
\end{theorem}

\section{Experiments} \label{sec:xp}

We assess the performance of the introduced algorithm through three batches of experiments. First, we use simulated data, where the theoretical importance values are known by construction, to compare our algorithm to the existing competitors. Secondly, we test our procedure with the semi-synthetic cases of the ACIC data challenge $2019$, where the variables involved in the heterogeneity are known, but not the importance value. Finally, we present cases with real data to show examples of an analysis conducted with our procedure.
Our approach is compared to the importance of the \texttt{grf} package  and TE-VIM, the double robust approach of \citet{hines2022variable}. For TE-VIM, any learning method can be used, and we report the performance of GAM models, which outperform regression forests in the presented experiments. When reading the results, recall that TE-VIM targets the same theoretical quantities $\VI$ as our algorithm, whereas the \texttt{grf} importance is the frequency of variable occurrence in tree splits. 
Besides, the algorithm of \citet{boileau2022flexible} is designed for high dimensional cases and linear treatment effects, and is thus not appropriate to our goal of precisely quantifying variable importance in non-linear settings. The implementation of our variable importance algorithm is available online at \url{https://gitlab.com/cbenard/grf-vimp}, along with the code to reproduce experiments with simulated data.

\subsection{Simulated Data}

\paragraph{Experiment 1.}
We consider a first example of simulated data to highlight the good performance of the proposed importance measure.
The input is of dimension $p = 8$, and is defined by $\bX \sim \mathcal{N}(\mathbf{0}, \Sigma)$, with $\Sigma$ the identity matrix except that $\mathrm{Cov}(X^{(1)}, X^{(5)}) = 0.9$. The treatment assignment is given by $W \sim \mathrm{Bernouilli}(0.4 + 0.2\mathds{1}_{X^{(1)} > 0})$, and the response $Y$ follows
\begin{align} \label{eq:model_xp1}
    Y = \big(X^{(1)}\mathds{1}_{X^{(1)} > 0} + 0.6X^{(2)}\mathds{1}_{X^{(2)} > 0}\big) \times W + (X^{(3)} \times X^{(4)})^2 + \varepsilon,
\end{align}
where $\varepsilon \sim \mathcal{N}(0, 0.1)$.
In practice, we take a sample size $n = 3000$, and the causal forest is fit with the default number of trees $M = 2000$. Notice that the ratio $\V[\tau(\bX^{(\bC)})]/\V[Y]$ is about $5\%$ in this setting, because of the high variance of the term $(X^{(3)} \times X^{(4)})^2$. Such a quite small ratio is realistic, and makes the treatment effect quite difficult to estimate in practice.  Here, both $X^{(1)}$ and $X^{(2)}$ are involved in heterogeneity, i.e. $\bC = \{1,2\}$, but only $X^{(1)}$ is also a confounder.  Results are averaged over $10$ repetitions, and are reported in Table \ref{table:xp_sim_1} ($30$ repetition for grf-vimp to stabilize the ranking). Additionally, the standard deviation of the mean importance for each variable is displayed in brackets, except for negligible values ($<0.005$). The first column of Table \ref{table:xp_sim_1} is the oracle importance value, precisely estimated using equation (\ref{eq:th_vimp}), the closed-form of $\tau$ given by equation (\ref{eq:model_xp1}), and a Monte-Carlo method with a large sample drawn from the joint distribution of $(Y,W,\bX)$, known by construction.
\begin{table}[h]
\renewcommand{\arraystretch}{0.9}
\centering
\begin{tabular}{|l |l |}
  \hline \hline
  \multicolumn{2}{|c|}{$\mathrm{I}$} \\
  \hline
  $\textcolor{blue}{X^{(2)}}$ & 0.26 \\
  $\textcolor{green}{X^{(1)}}$ & 0.18 \\
  $X^{(3)}$ & 0 \\
  $X^{(4)}$ & 0 \\
  $X^{(5)}$ & 0 \\
  $X^{(6)}$ & 0 \\
  $X^{(7)}$ & 0 \\
  $X^{(8)}$ & 0 \\
  \hline \hline
\end{tabular}
\begin{tabular}{|l |l |}
  \hline \hline
  \multicolumn{2}{|c|}{$\mathrm{I}_n$} \\
  \hline
  $\textcolor{blue}{X^{(2)}}$ & 0.23 \tiny{(0.02)} \\
  $\textcolor{green}{X^{(1)}}$ & 0.19 \tiny{(0.01)} \\
  $X^{(4)}$ & 0.04 \tiny{(0.01)}\\
  $X^{(3)}$ & 0.03 \tiny{(0.01)}\\
  $X^{(5)}$ & 0.004 \\
  $X^{(6)}$ & 0.001 \\
  $X^{(7)}$ & 0.001 \\
  $X^{(8)}$ & 0.001 \\
  \hline \hline
\end{tabular}
\begin{tabular}{|l |l |}
  \hline \hline
  \multicolumn{2}{|c|}{TE-VIM} \\
  \hline
  $\textcolor{green}{X^{(1)}}$ & 0.42 \tiny{(0.07)} \\
  $\textcolor{blue}{X^{(2)}}$ & 0.40 \tiny{(0.08)} \\
  $X^{(4)}$ & 0.19 \tiny{(0.32)} \\
  $X^{(8)}$ & 0.14 \tiny{(0.16)} \\
  $X^{(5)}$ & 0.14 \tiny{(0.15)} \\
  $X^{(3)}$ & 0.12 \tiny{(0.19)} \\
  $X^{(6)}$ & 0.05 \tiny{(0.15)} \\
  $X^{(7)}$ & -0.01 \tiny{(0.17)} \\
  \hline \hline
\end{tabular}
\begin{tabular}{|l |l |}
  \hline \hline
  \multicolumn{2}{|c|}{grf-vimp} \\
  \hline
  $\textcolor{green}{X^{(1)}}$ & 0.49 \tiny{(0.02)}\\
  $X^{(3)}$ & 0.13 \tiny{(0.01)}\\
  $X^{(4)}$ & 0.12 \tiny{(0.01)}\\
  $X^{(5)}$ & 0.11 \tiny{(0.01)}\\
  $\textcolor{blue}{X^{(2)}}$ & 0.10 \tiny{(0.01)}\\
  $X^{(6)}$ & 0.02 \\
  $X^{(7)}$ & 0.02 \\
  $X^{(8)}$ & 0.02 \\
  \hline \hline
\end{tabular}
\caption{Variable importance ranking of Experiment $1$ for $\VIhat$, the importance measure of \texttt{grf} package, and TE-VIM. Standard deviations are displayed in brackets when greater than $0.005$.}
\label{table:xp_sim_1}
\end{table}

The results displayed in Table \ref{table:xp_sim_1} show that our algorithm is the only one to provide the accurate variable ranking, where $X^{(2)}$ is the most important variable, and $X^{(1)}$ the second most important one. TE-VIM accurately identifies these two variables as the most influential, with a similar importance. On the other hand, the importance measure from the \texttt{grf} package underestimates the importance of variable $X^{(2)}$, and identifies $X^{(3)}$, $X^{(4)}$, and $X^{(5)}$ as slightly more important than $X^{(2)}$, although these three variables are not involved in the treatment heterogeneity by construction. In particular, $X^{(5)}$ is not involved at all in the response $Y$, but is strongly correlated to the influential input $X^{(1)}$. Because of this dependence, $X^{(5)}$ is frequently used in the causal forests splits, leading to this quite high importance given by the \texttt{grf} package. On the other hand, $\smash{\VIhat}$ gives an importance close to $0$ for $X^{(5)}$. This result is expected, since the removal of $X^{(5)}$ does not lead to any loss of information regarding the treatment heterogeneity, by definition. An additional interesting phenomenon is the non-negligible importance for variables $X^{(3)}$ and $X^{(4)}$ given by all procedures. In fact, the interaction term in the baseline function $\mu$, which takes the form of a squared product, is rather difficult to estimate by regression forests. Then, the local centering of $Y$ is only partial, and $X^{(3)}$ and $X^{(4)}$ still have impact on the variance of treatment estimates. 
Besides, notice that the corrective term of equation (\ref{eq:correction}) is negligible in this experiment, and that using the original causal forest retrained with one variable removed, gives the same result as in Table \ref{table:xp_sim_1} for $\smash{\VIhat}$, up to the displayed digits.

\paragraph{Experiment 2.}
This second experiment has the same setting than Experiment $1$, except that variable $X^{(1)}$ is only a confounder and is not involved in the treatment effect heterogeneity anymore. Now, the response writes
\begin{align*}
    Y = \big(0.6X^{(2)}\mathds{1}_{X^{(2)} > 0}\big) \times W + X^{(1)}\mathds{1}_{X^{(1)} > 0} + (X^{(3)} \times X^{(4)})^2 + \varepsilon.
\end{align*}
The results are provided in Table \ref{table:xp_sim_2}.
Clearly, $\VIhat$ outperforms the competitors. Indeed, $X^{(2)}$ is well-identified by $\VIhat$ as responsible for most of the heterogeneity of the treatment effect, whereas TE-VIM is strongly biased, and the importance procedure of the \texttt{grf} package outputs quite close values for $X^{(2)}$, $X^{(4)}$, and $X^{(3)}$. As expected, the importance of these last two variables is relatively larger than in Experiment $1$, since the ratio $\V[\tau(\bX^{(\bC)})]/\V[Y]$ drops to $1\%$ in this case.
\begin{table}[h]
\renewcommand{\arraystretch}{0.9}
\centering
\begin{tabular}{|l |l |}
  \hline \hline
  \multicolumn{2}{|c|}{$\mathrm{I}$} \\
  \hline
  $\textcolor{blue}{X^{(2)}}$ & 1 \\
  $X^{(1)}$ & 0 \\
  $X^{(3)}$ & 0 \\
  $X^{(4)}$ & 0 \\
  $X^{(5)}$ & 0 \\
  $X^{(6)}$ & 0 \\
  $X^{(7)}$ & 0 \\
  $X^{(8)}$ & 0 \\
  \hline \hline
\end{tabular}
\hspace*{3mm}
\begin{tabular}{|l |l |}
  \hline \hline
  \multicolumn{2}{|c|}{$\mathrm{I}_n$} \\
  \hline
  $\textcolor{blue}{X^{(2)}}$ & 0.89 \tiny{(0.04)} \\
  $X^{(3)}$ & 0.13 \tiny{(0.03)}\\
  $X^{(4)}$ & 0.13 \tiny{(0.03)}\\
  $X^{(1)}$ & 0.003 \\
  $X^{(5)}$ & 0.003 \\
  $X^{(6)}$ & 0.004 \\
  $X^{(7)}$ & 0.004 \\
  $X^{(8)}$ & 0.006 \\
  \hline \hline
\end{tabular}
\begin{tabular}{|c |l |}
  \hline \hline
  \multicolumn{2}{|c|}{TE-VIM} \\
  \hline
  $\textcolor{blue}{X^{(2)}}$ & 1.76 \tiny{(0.11)} \\
  $X^{(4)}$ & 1.65 \tiny{(0.04)} \\
  $X^{(3)}$ & 1.03 \tiny{(0.02)} \\
  $X^{(8)}$ & 0.99 \\
  $X^{(1)}$ & 0.96 \tiny{(0.02)} \\
  $X^{(5)}$ & 0.88 \tiny{(0.02)} \\
  $X^{(6)}$ & 0.71 \tiny{(0.03)} \\
  $X^{(7)}$ & 0.57 \tiny{(0.04)} \\
  \hline \hline
\end{tabular}
\begin{tabular}{|l |l |}
  \hline \hline
  \multicolumn{2}{|c|}{grf-vimp} \\
  \hline
  $\textcolor{blue}{X^{(2)}}$ & 0.36 \tiny{(0.01)}\\
  $X^{(4)}$ & 0.24 \tiny{(0.01)}\\
  $X^{(3)}$ & 0.23 \tiny{(0.01)}\\
  $X^{(1)}$ & 0.03 \\
  $X^{(5)}$ & 0.03 \\
  $X^{(6)}$ & 0.03 \\
  $X^{(7)}$ & 0.03 \\
  $X^{(8)}$ & 0.03 \\
  \hline \hline
\end{tabular}
\caption{Variable importance ranking of Experiment $2$ for $\VIhat$, the importance measure of \texttt{grf} package, and TE-VIM. Standard deviations are displayed in brackets when greater than $0.005$.}
\label{table:xp_sim_2}
\end{table}

\paragraph{Experiment 3.}
The goal of this third simulated experiment is to highlight a case where the corrective term in the retrained causal forest has a strong influence, as opposed to Experiments $1$ and $2$. We consider $p=5$ inputs uniformly distributed over $[0,1]$, except $X^{(1)}$ defined as $X^{(1)} = U^3$, where $U \sim \mathcal{U}(0,1)$. The treatment assignment $W$ is a Bernoulli variable defined from $\pi(\bX) = X^{(1)}$, and the response is given by 
\begin{align*}
    Y = 10 X^{(1)}(1 - X^{(1)}) \times W + X^{(2)} + \varepsilon,
\end{align*}
where $\varepsilon \sim \mathcal{N}(0, 0.1)$. We still use $n = 3000$ and $M = 2000$ trees in the causal forests. Next, we compute our importance measure $\smash{\VIhat}$ for all inputs, as well as its counterpart $\smash{\mathcal{I}_n^{(j)}}$, where the corrective term is removed, and with $10$ repetitions for uncertainties. Results are reported in Table \ref{table:xp_sim_3}, and clearly show the high bias of the importance of $\smash{X^{(1)}}$ when the corrective term in the retrained forest is removed. Indeed, we get $\smash{\mathcal{I}_n^{(1)} = 1.57}$, whereas the target quantity is $\smash{\mathrm{I}^{(1)} = 1}$, since $\smash{X^{(1)}}$ is the only variable involved in the treatment effect heterogeneity and $\smash{X^{(1)}}$ is independent of the other inputs. With the correction, we recover an importance value of $0.98$ for $\smash{X^{(1)}}$ as expected. Notice that the asymptotic bias exhibited in Theorem \ref{thm:vimp_consistency_bias} takes values $0.72$ for this case, which explains the empirical results. Importantly, this bias takes small values in practice in most cases. Here, we take the treatment effect as $\smash{\tau(\bX^{(\bC)}) = 10 \pi(\bX)(1 - \pi(\bX))}$ to maximize the covariance term involved in the bias of Theorem \ref{thm:vimp_consistency_bias}.
\begin{table}[h]
\renewcommand{\arraystretch}{0.9}
\centering
\begin{tabular}{|l |l |}
  \hline \hline
  \multicolumn{2}{|c|}{$\mathrm{I}_n$} \\
  \hline
  $\textcolor{blue}{X^{(1)}}$ & $0.98$ \tiny{(0.002)} \\
  $X^{(2)}$ & $0.0003$ \\
  $X^{(3)}$ & $0.001$ \\
  $X^{(4)}$ & $0.0002$ \\
  $X^{(5)}$ & $0.0002$ \\
  \hline \hline
\end{tabular}
\begin{tabular}{|c |l |}
  \hline \hline
  \multicolumn{2}{|c|}{$\mathcal{I}_n$} \\
  \hline
  $\textcolor{blue}{X^{(1)}}$ & $1.57$ \tiny{(0.01)} \\
  $X^{(2)}$ & $0.001$ \\
  $X^{(3)}$ & $0.001$ \\
  $X^{(4)}$ & $0.002$ \\
  $X^{(5)}$ & $0.001$ \\
  \hline \hline
\end{tabular}
\caption{Variable importance ranking of Experiment $3$ for $\VIhat$ and $\mathcal{I}_n^{(j)}$. Standard deviations are displayed in brackets when greater than $0.001$.}
\label{table:xp_sim_3}
\end{table}


\subsection{ACIC Data Challenge 2019}
We run a second batch of experiments using the data from the ACIC data challenge $2019$ (\url{https://sites.google.com/view/acic2019datachallenge/data-challenge}), where the goal was to estimate ATEs in various settings. The input data is taken from real datasets available online on the UCI repository. Next, outcomes are simulated with different scenarios, and the associated code scripts were released after the challenge. Since the data generating mechanism is available, we have access to the variables involved in the heterogeneous treatment effect. In each scenario, a hundred datasets were randomly sampled.

\begin{table}[h]
\centering
\begin{tabular}{|l |l |}
  \hline \hline
  \multicolumn{2}{|c|}{$\mathrm{I}_n^{(j)}$} \\
  \hline
  $\textcolor{blue}{X^{(3)}}$ & 0.82 \tiny{$(0.04)$} \\
  $X^{(27)}$ & 0.009 \tiny{$(0.009)$}\\
  $X^{(29)}$ & 0.008 \tiny{$(0.003)$}\\
  $X^{(12)}$ & 0.007 \tiny{$(0.005)$}\\
  $X^{(14)}$ & 0.005 \tiny{$(0.004)$}\\
  \hline \hline
\end{tabular}
\begin{tabular}{|l |l |}
  \hline \hline
  \multicolumn{2}{|c|}{grf-vimp} \\
  \hline
  $\textcolor{blue}{X^{(3)}}$ & 0.45 \tiny{$(0.04)$} \\
  $X^{(29)}$ & 0.06 \tiny{$(0.008)$}\\
  $X^{(27)}$ & 0.03 \tiny{$(0.008)$}\\
  $X^{(7)}$ & 0.03 \tiny{$(0.005)$}\\
  $X^{(28)}$ & 0.03 \tiny{$(0.002)$}\\
  \hline \hline
\end{tabular}
\caption{Top $5$ variables for ``Student performance 2 (Scenario $4$)'' dataset using $\VIhat$ and the importance measure of \texttt{grf} package. Standard deviations are displayed in brackets.}
\label{table:xp_ACIC_student}
\end{table}
We first use the ``student performance 2'' data with $31$ input variables, considering Scenario $4$ defined in the ACIC challenge, involving heterogeneity of the treatment effect with respect to $X^{(3)}$. Each dataset is of size $n = 649$, and we run $10$ repetitions with independent datasets for uncertainties. Table \ref{table:xp_ACIC_student} gives the top $5$ variables ranked by $\VIhat$, which accurately identifies $X^{(3)}$ as the only variable involved in the treatment heterogeneity, since other variables all have a negligible importance value. The \texttt{grf} importance measure also identifies $X^{(3)}$ as the most important variable. However, the importance of many irrelevant variables is not negligible, as opposed to $\VIhat$.

Secondly, we use the ``spam email'' data, made of $22$ input variables. We also consider Scenario $4$, where variables $X^{(8)}$ and $X^{(19)}$ are involved in the heterogeneous treatment effect. In this case, we merge $20$ datasets to get a quite large sample of size $n = 10 000$, and run $5$ repetitions to compute standard deviations. The two relevant variables are properly identified as the most important ones by the two tested algorithms, as shown in Table \ref{table:xp_ACIC_spam}.  Again, the \texttt{grf} importance gives slightly higher values to irrelevant variables than $\smash{\VIhat}$. Notice that the impact of $X^{(19)}$ on heterogeneity is really small, and if we use only few datasets of size  $n = 500$ in the forest training, $X^{(19)}$ is not identified as more important than noisy variables. Thus, a large sample size is required to detect its influence, and therefore we use $n = 10000$.
\begin{table}[h]
\centering
\begin{tabular}{|l |l |}
  \hline \hline
  \multicolumn{2}{|c|}{$\mathrm{I}_n$} \\
  \hline
  $\textcolor{blue}{X^{(8)}}$ & 0.83 \tiny{$(0.001)$} \\
  $\textcolor{blue}{X^{(19)}}$ & 0.011 \tiny{$(0.002)$} \\
  $X^{(22)}$ & 0.003 \tiny{$(4.10^{-4})$}\\
  $X^{(12)}$ & 0.002 \tiny{$(4.10^{-4})$}\\
  $X^{(15)}$ & 0.001 \tiny{$(3.10^{-4})$}\\
  $X^{(17)}$ & 0.0004 \tiny{$(< 10^{-4})$}\\
  \hline \hline
\end{tabular}
\begin{tabular}{|l |l |}
  \hline \hline
  \multicolumn{2}{|c|}{grf-vimp} \\
  \hline
  $\textcolor{blue}{X^{(8)}}$ & 0.85 \tiny{$(4.10^{-3})$} \\
  $\textcolor{blue}{X^{(19)}}$ & 0.064 \tiny{$(6.10^{-3})$} \\
  $X^{(1)}$ & 0.013 \tiny{$(3.10^{-3})$}\\
  $X^{(22)}$ & 0.013 \tiny{$(1.10^{-3})$}\\
  $X^{(15)}$ & 0.010 \tiny{$(8.10^{-4})$}\\
  $X^{(17)}$ & 0.009 \tiny{$(2.10^{-3})$}\\
  \hline \hline
\end{tabular}
\caption{Top $6$ variables for ``Spam email (Scenario $4$)'' dataset using $\VIhat$ and the importance measure of \texttt{grf} package. Standard deviations are displayed in brackets.}
\label{table:xp_ACIC_spam}
\end{table}

\subsection{Real data}

\paragraph{Welfare data.}
For a first experiment with real data, we use the ``Welfare'' dataset from a GSS survey, introduced in \citet{green2012modeling} and available at \url{https://github.com/gsbDBI/ExperimentData}. The goal of this survey is to analyze the impact of question wording about the support of Americans to the government welfare spending. Respondents are randomly assigned one of two possible questions, with the same introduction and response options, but using the phrasing ``welfare'' or ``assistance to the poor''. In fact, this slight wording difference has a quite strong impact on the survey answers, and defines the treatment. The output of interest indicates if respondents have answered that ``too much'' is spent. Our objective is to identify the main characteristics of individuals that have an impact on the heterogeneity of the treatment effect.
We take the dataset from the tutorial available at \url{https://gsbdbi.github.io/ml_tutorial/hte_tutorial/hte_tutorial.html}, of size $n = 13198$ and with $p = 31$ input variables, where basic data preparation steps were used to drop rows with missing values. Notice that we consider the same data to enable comparisons, but that imputing missing values may improve estimates. We leave this topic for future work, as handling missing values for variable importance is of high practical interest.

Table \ref{table:welfare} displays the top $10$ most important variables for Welfare data using our algorithm $\mathrm{I}_n$ and also the importance from the \texttt{grf} package. 
The ranking provided by the two algorithms are close, but $\mathrm{I}_n$ has a clear meaning as the variance proportion of the treatment effect lost when a given variable is removed, whereas grf-vimp can only be used as a relative importance between covariates, without an intrinsic meaning.
\begin{table}[h]
\centering
\begin{tabular}{|c|c|}
  \hline \hline
\multicolumn{2}{|c|}{$\mathrm{I}_n$}  \\ 
  \hline
  polviews & 0.18 \\ 
  partyid & 0.09 \\ 
  hrs1 & 0.04 \\ 
  indus80 & 0.03 \\ 
  maeduc & 0.02 \\ 
  educ & 0.02 \\
  marital & 0.01 \\ 
  age & 0.01 \\ 
  occ80 & 0.01 \\ 
  reg16 & 0.01 \\ 
   \hline \hline
\end{tabular}
\begin{tabular}{|c|c|}
  \hline \hline
\multicolumn{2}{|c|}{grf-vimp} \\ 
  \hline
  polviews & 0.31 \\ 
  partyid & 0.17 \\ 
  educ & 0.09 \\ 
  indus80 & 0.07 \\ 
  hrs1 & 0.07 \\ 
  marital & 0.04 \\ 
  degree & 0.04 \\ 
  maeduc & 0.04 \\ 
  occ80 & 0.02 \\ 
  age & 0.02 \\ 
   \hline \hline
\end{tabular}
\caption{Top $10$ most important variables with respect to $\mathrm{I}_n$ and grf-vimp for Welfare data.}
\label{table:welfare}
\end{table}

Notice that the sum of the importance of all input variables, i.e. $\sum_j \VIhat$, adds to $0.45$, which is far from $1$. Indeed, when inputs are independent, we have $\smash{\sum_j \VI \geq 1}$. 
Such a low value is explained by the correlation within input variables. We run a simple hierarchical clustering of the input variables in $10$ groups based on correlation, to enforce a small correlation between these groups. Then, we run the group variable importance $\smash{\mathrm{I}_n^{(J)}}$ for each group of variables $J \subset \{1,\hdots,p\}$. The results are displayed in the following Table \ref{table:welfare_group}, and are quite straightforward to read. Indeed, half of the treatment heterogeneity is explained by political orientations of individuals, almost a quarter of the heterogeneity is given by variables mostly related to education and degrees. Then, several groups have a small impact, especially a group about income and working status, and a second one about family information. 
\begin{table}[h]
\centering
\begin{tabular}{|c|c|}
  \hline \hline
  Variable group & $I_n^{(J)}$ \\ 
  \hline
  partyid, polviews & 0.51 \\ 
  educ, sibs, occ80, prestg80, maeduc, degree & 0.23 \\ 
  hrs1, income, rincome, wrkstat & 0.07 \\ 
  age, marital, childs, babies & 0.04 \\ 
  wrkslf, indus80, sex & 0.03 \\ 
  reg16, mobile16 & 0.01 \\ 
  race, res16, parborn, born & 0.00 \\ 
  family16 & 0.00 \\ 
  earnrs, hompop, adults & 0.00 \\ 
  preteen, teens & 0.00 \\ 
  \hline \hline
\end{tabular}
\caption{Group variable importance for Welfare data.}
\label{table:welfare_group}
\end{table}

\paragraph{NHEFS health data.}
For the second case study, we use the NHEFS real data about body weight gain following a smoking cessation, extensively described in the causal inference book of \citet{hernan2020causal}. As highlighted in the introduction of Chapter $12$, these data help to answer the question ``what is the average causal effect of smoking cessation on body weight gain?''. According to the authors, the unconfoundedness assumption holds. Here, we go a step further to analyze the heterogeneity of this causal effect with respect to health and personal data of individuals who have stopped smoking, using causal forests and our variable importance algorithm. The data record the weight of individuals, first measured in $1971$, and then in $1982$. The treatment assignment $W$ indicates whether people have stopped smoking during this period, and the observed output $Y$ is the weight difference between $1971$ and $1982$. We take the dataset of size $n = 1566$ used in \citet[Chapter 12]{hernan2020causal}. Notice that $63$ rows with the output missing were removed, introducing a small bias, as discussed by the authors. They include $9$ variables in their analysis, sufficient for unconfoundedness. To better estimate heterogeneity, we also include all variables of the original dataset, that do not contain missing values and are not related to the response, and obtain $p = 41$ input variables. As already mentioned, handling missing values is out of scope of this article, and is left for future work. We run our variable importance algorithm and the grf importance, using $M = 4000$ trees.

The results are displayed in Table \ref{table:nhefs}. Clearly, the original weight of individuals in $1971$ has a strong causal effect on weight gain following smoking cessation, with half of the treatment effect variance lost when this variable is removed. 
The intensity and duration of smoking, as well as personal characteristics, such as height and age are also involved in treatment heterogeneity, according to both algorithms. Notice that \texttt{grf} importance underestimates the importance of wt71 with respect to other variables. Next, we group together variables that are highly correlated, to compute group variable importance. Sex, height, and birth control are highly correlated with the weight in $1971$, and this group explains two third of the treatment effect heterogeneity. In fact, age and smoke years also have a quite strong impact with a quarter of heterogeneity explained.

\begin{table}[h]
\centering
\begin{tabular}{|c|c|}
  \hline \hline
\multicolumn{2}{|c|}{$\mathrm{I}_n$}  \\ 
  \hline
  wt71 & 0.52 \\ 
  smokeyrs & 0.09 \\ 
  smokeintensity & 0.07 \\ 
  ht & 0.06 \\ 
  age & 0.05 \\ 
  alcoholfreq & 0.01 \\ 
  active & 0.01 \\ 
  tumor & 0.01 \\ 
  asthma & 0.01 \\ 
  alcoholtype & 0.01 \\ 
   \hline \hline
\end{tabular}
\begin{tabular}{|c|c|}
  \hline \hline
\multicolumn{2}{|c|}{grf-vimp} \\ 
  \hline
  wt71 & 0.26 \\ 
  smokeyrs & 0.13 \\ 
  age & 0.10 \\ 
  ht & 0.10 \\ 
  smokeintensity & 0.07 \\ 
  school & 0.07 \\ 
  active & 0.03 \\ 
  alcoholfreq & 0.03 \\ 
  chroniccough & 0.02 \\ 
  marital & 0.02 \\ 
   \hline \hline
\end{tabular}
\caption{Top $10$ most important variables with respect to $\mathrm{I}_n$ and grf-vimp for NHEFS data.}
\label{table:nhefs}
\end{table}

\begin{table}[h]
\centering
\begin{tabular}{|c|c|}
  \hline \hline
  Variable group & $I_n^{(J)}$ \\ 
  \hline
  sex, ht, wt71, birthcontrol & 0.67 \\ 
  age, smokeyrs & 0.26 \\ 
  school, education & 0.03 \\ 
  alcoholpy, alcoholfreq, alcoholtype & 0.02 \\ 
  hbp, diabetes, pica, hbpmed, boweltrouble & 0.02 \\ 
  \hline \hline
\end{tabular}
\caption{Group variable importance for NHEFS data.}
\label{table:nhefs_group}
\end{table}

\section{Conclusion}

We introduced a new variable importance algorithm for causal forests, based on the drop and relearn principle, widely used for regression problems. The proposed method has both theoretical and empirical solid groundings. Indeed, we show that our algorithm is consistent, under standard assumptions in the mathematical analysis of random forests. Additionally, we run extensive experiments on simulated, semi-synthetic, and real data, to show the practical efficiency of the method. Notice that the implementation of our variable importance algorithm is available online at \url{https://gitlab.com/cbenard/grf-vimp}.

Let us summarize the main guidelines for practitioners using our variable importance algorithm. First, all confounders must be included in the initial data, as it is always necessary to fulfill the unconfoundedness assumption to obtain consistent estimates. Secondly, it is also recommended to include all variables impacting heterogeneity in the data as well. However, leaving aside a non-confounding variable impacting heterogeneity, does not bias the analysis, as opposed to a missing confounder. Thirdly, practitioners must also keep in mind that adding a large number of irrelevant variables, i.e. non-confounding and not impacting heterogeneity, may hurt the accuracy of causal forests. Finally, it is recommended to group correlated variables together, and then compute group variable importance to get additional relevant insights.

To conclude, we mention two topics of high interest for future work.
First, handling missing values in variable importance algorithms is barely discussed in the literature, but is strongly useful in practice, since observational databases often have missing values, which should be handled carefully to avoid misleading results.
Secondly, developing a testing procedure to detect significantly non-null importance values, would enable to identify the set $\bC$ of variables involved in heterogeneity, an insight of high practical value. The asymptotic normality of causal forests is probably a promising starting point to develop such testing algorithms.


\FloatBarrier

\bibliography{biblio}

\newpage

\appendix

\section{Proofs of Propositions \ref{prop_identifiability}-\ref{prop:local_moment_j} and Theorems \ref{thm:cf_consistency}-\ref{thm:vimp_consistency_bias}} \label{App:proofs}

\begin{proof}[Proof of Proposition \ref{prop_identifiability}]
    Using the observed outcome definition with SUTVA (line $1$), and the unconfoundedness Assumption \ref{A:unconfound} (line $2$ to $3$), we have
    \begin{align*}
        \E[Y \mid \bX, W] & = \E[W Y(1) + (1 - W) Y(0) \mid \bX, W] \\
        & = W \E[Y(1) \mid \bX, W] + (1 - W) \E[Y(0) \mid \bX, W] \\
        & = W \E[Y(1) \mid \bX] + (1 - W) \E[Y(0) \mid \bX] \\
        & = \E[Y(0) \mid \bX] + W (\E[Y(1) \mid \bX] - \E[Y(0) \mid \bX]) \\
        & = \E[Y(0) \mid \bX] + W \E[Y(1) - Y(0) \mid \bX]) \\
        & = \E[\mu(\bX) + \varepsilon(0) \mid \bX] + W \E[\tau(\bX^{(\bC)}) + \varepsilon(1) - \varepsilon(0) \mid \bX]) \\
        & = \mu(\bX) + W \tau(\bX^{(\bC)}),
    \end{align*}
    and the final result follows.
\end{proof}

\begin{proof}[Proof of Proposition \ref{prop:heterogeneous_prob}]
    From Assumption \ref{A:tau}, $\bX$ admits a strictly positive density, denoted by $f$. Then, from Definition \ref{Def:tau},
    \begin{align*}
        \P(\tau(\bX^{(\bC)}) \neq \tau(\bX'^{(\bC)})) > \int_{\mathcal{X}_1 \times \mathcal{X}'_1 \times \mathcal{X}_{p-1}} f(x^{(j)}, \bx^{(-j)}) f(x'^{(j)}, \bx^{(-j)}) dx^{(j)} dx'^{(j)} d\bx^{(-j)},
    \end{align*}
    which is strictly positive, since f is strictly positive and $\mathcal{X}_1$, $\mathcal{X}'_1$, and $\mathcal{X}_{p-1}$ have a non-null Lebesgue measure.
\end{proof}

\begin{proof}[Proof of Proposition \ref{prop_VI}]
    Assumption \ref{A:tau} implies that $\V[\tau(\bX^{(\bC)})] > 0$.
    By definition, 
    \begin{align} \label{eq:vimp}
        \VI = \frac{\V[\tau(\bX^{(\bC)})] - \V[\E[\tau(\bX^{(\bC)})|\bX^{(-j)}]]}{\V[\tau(\bX^{(\bC)})]},
    \end{align}
    which also writes using the law of total variance
    \begin{align} \label{eq:total_var}
        \VI = \frac{\E[\V[\tau(\bX^{(\bC)})|\bX^{(-j)}]]}{\V[\tau(\bX^{(\bC)})]}
        = \frac{\E[(\tau(\bX^{(\bC)}) - E[\tau(\bX^{(\bC)})|\bX^{(-j)}])^2]}{\V[\tau(\bX^{(\bC)})]}.
    \end{align}
    If $j \notin \bC$, we clearly have $E[\tau(\bX^{(\bC)})|\bX^{(-j)}] = \tau(\bX^{(\bC)})$, and then equation (\ref{eq:total_var}) gives that $\VI = 0$.

    We now consider the case where $j \in \bC$. 
    First, since $\V[\E[\tau(\bX^{(\bC)})|\bX^{(-j)}]] \geq 0$, we directly get that $\VI \leq 1$ from equation (\ref{eq:vimp}).
    Secondly, from Definition \ref{Def:tau}, for $\bx^{(-j)} \in \mathcal{X}_{p-1}$, the function $x^{(j)} \to \tau(x^{(j)}, \bx^{(-j)})$ takes different values over $\mathcal{X}_1$ and $\mathcal{X}'_1$, and therefore $(\tau(\bX^{(\bC)}) - E[\tau(\bX^{(\bC)})|\bX^{(-j)}])^2 > 0$ with a positive probability, since $\mathcal{X}_1$, $\mathcal{X}'_1$, and $\mathcal{X}_{p-1}$ have a non-null Lebesgue measure. It implies that $\VI > 0$.
\end{proof}

\begin{proof}[Proof of Proposition \ref{prop:local_moment_C}]
    We first expand the covariance term
    \begin{align*}
        \mathrm{Cov}[&W - \pi(\bX), Y - m(\bX) \mid \bX^{(\bC)}] \\
        &= \E[(W - \pi(\bX))(Y - m(\bX)) \mid \bX^{(\bC)}] - \E[W - \pi(\bX)\mid \bX^{(\bC)}]\E[Y - m(\bX) \mid \bX^{(\bC)}].
    \end{align*}
    Notice that the second term is null since $\E[Y - m(\bX) \mid \bX^{(\bC)}] = \E[\E[Y - m(\bX) \mid \bX] \mid \bX^{(\bC)}] = 0$. Additionally, by definition,
    \begin{align*}
        m(\bX) = \E[Y \mid \bX] & = \E[\mu(\bX) + \tau(\bX^{(\bC)}) \times W + \varepsilon(W) \mid \bX] = \mu(\bX) + \tau(\bX^{(\bC)}) \pi(\bX),
    \end{align*}
    then $Y - m(\bX) = (W - \pi(\bX)) \tau(\bX^{(\bC)}) + \varepsilon(W)$, and we get
    \begin{align*}
        \mathrm{Cov}[W - \pi(\bX)&, Y - m(\bX) \mid \bX^{(\bC)}] \\
        &= \E[(W - \pi(\bX))((W - \pi(\bX)) \tau(\bX^{(\bC)}) + \varepsilon(W)) \mid \bX^{(\bC)}] \\
        &= \tau(\bX^{(\bC)}) \times \E[(W - \pi(\bX))^2 \mid \bX^{(\bC)}] + \E[\varepsilon(W) (W - \pi(\bX)) \mid \bX^{(\bC)}] \\
        &= \tau(\bX^{(\bC)}) \times \E[(W - \pi(\bX))^2 \mid \bX^{(\bC)}] + \E[ (W - \pi(\bX)) \E[\varepsilon(W)  \mid \bX, W] \mid \bX^{(\bC)}]] \\
        &= \tau(\bX^{(\bC)}) \times \V[W - \pi(\bX) \mid \bX^{(\bC)}],
    \end{align*}
    which gives the final local moment equation in $\bX^{(\bC)}$.
\end{proof}

\begin{proof}[Proof of Proposition \ref{prop:local_moment_j}]
    As in the proof of Proposition \ref{prop:local_moment_C}, we obtain
    \begin{align*}
        \mathrm{Cov}[W - \pi(\bX), Y - m(\bX) \mid \bX^{(-j)}] 
            &= \E[\tau(\bX^{(\bC)})(W - \pi(\bX))^2 \mid \bX^{(-j)}].
    \end{align*}
    Notice that
    \begin{align*}
        \mathrm{Cov}[\tau(\bX^{(\bC)}), (W - \pi(\bX))^2 \mid \bX^{(-j)}] = \E[\tau(&\bX^{(\bC)})(W - \pi(\bX))^2 \mid \bX^{(-j)}] \\ &- \E[\tau(\bX^{(\bC)})\mid \bX^{(-j)}]\E[(W - \pi(\bX))^2 \mid \bX^{(-j)}].
    \end{align*}
    Combining the above two equations, we have
    \begin{align*}
        \mathrm{Cov}[W - \pi(\bX), Y - m(\bX) \mid \bX^{(-j)}] 
            =& \mathrm{Cov}[\tau(\bX^{(\bC)}), (W - \pi(\bX))^2 \mid \bX^{(-j)}] \\ &+ \E[\tau(\bX^{(\bC)})\mid \bX^{(-j)}] \times \V[W - \pi(\bX) \mid \bX^{(-j)}],
    \end{align*}
    which gives the final result since 
    \begin{align*}
        \mathrm{Cov}[\tau(\bX^{(\bC)}), (W - \pi(\bX))^2 \mid \bX^{(-j)}] = \mathrm{Cov}[\tau(\bX^{(\bC)}), \pi(\bX)(1 - \pi(\bX)) \mid \bX^{(-j)}].
    \end{align*}
\end{proof}

\begin{proof}[Proof of Theorem \ref{thm:cf_consistency}]
    The result is obtained by applying Theorem $3$ from \citet{athey2019generalized}. The first paragraph of section $3$ of \citet{athey2019generalized} provides conditions to apply Theorem $3$, that are satisfied by our Assumptions \ref{A:X_density} and \ref{A:lipschitz}: $\bX \in [0,1]^p$, $\bX$ admits a density bounded from below and above by strictly positive constants, and $\mu$ and $\tau$ are bounded.
    
    Next, Assumptions 1-6 from \citet{athey2019generalized} must be verified. As stated at the end of Section $6.1$, Assumptions 3-6 always hold for causal forests, the first assumption holds because the functions $m$, $\mu$, and $\tau$ are Lispschitz from our Assumption \ref{A:lipschitz} (the product of Lipschitz functions is Lipschitz), and Assumption $2$ is satisfied because $0 < \V[W \mid \bX] = \pi(\bX)(1 - \pi(\bX)) < 1$ from our Assumption \ref{A:lipschitz}.

    Finally, the forest is grown from Specification \ref{Spec:forests}, and the treatment effect is identified by equation (\ref{eq:local_moment_full}) since Assumption \ref{A:unconfound} enforces unconfoundedness. Overall, we apply Theorem $3$ from \citet{athey2019generalized} to get the consistency of the causal forest estimate, i.e., for $\bx \in [0,1]^p$
    \begin{align*}
        \tau_{M,n}(\bx) \overset{p}{\longrightarrow} \tau(\bx^{(\bC)}).
    \end{align*}

    Notice that Theorem $3$ from \citet{athey2019generalized} states the consistency of generalized forests. As it will be useful for further results, we give below a proof of the weak consistency in the specific case of causal forests, using arguments of \citet{athey2019generalized}. In particular, we take advantage of Specification \ref{Spec:forests}, which enforces the honesty property, and that the diameters of tree cells vanish as the sample size $n$ increases.
    First, in our case of binary treatment $W$, the causal forest estimate writes
    \begin{align*}
        \tau_{M,n}(\bx) = \frac{\sum_{i=1}^n \alpha_i(\bx) W_i Y_i - (\sum_{i=1}^n \alpha_i(\bx) W_i) (\sum_{i=1}^n \alpha_i(\bx) Y_i)}{\sum_{i=1}^n \alpha_i(\bx) W_i^2 - (\sum_{i=1}^n \alpha_i(\bx) W_i)^2},
    \end{align*}
    where the weight $\alpha_i(\bx)$ is defined by equation ($3$) of \citet{athey2019generalized}, as the weight associated to training observation $\bX_i$ to form an estimate at the new query point $\bx$. The weights $\alpha_i(\bx)$ sum to $1$ over all observations, i.e., $\sum_{i=1}^n \alpha_i(\bx) = 1$. Also notice that we alleviate notations of $\alpha_i(\bx)$ throughout the article, but the full expression with all dependencies is $\alpha_i(\bx, \bX_i, \bTheta_M, \Dn)$, where the causal forest is built with data $\Dn$, and trees are randomized with $\bTheta_M$. Now, we denote by $\Delta_{1,n}(\bx) = \sum_{i=1}^n \alpha_i(\bx) W_i Y_i$ the first term of the numerator of $\tau_{M,n}(\bx)$, and derive its convergence. Since the weights sum to $1$,
    \begin{align*}
        \Delta_{1,n}(\bx) - \E[WY \mid \bX = \bx] = \sum_{i=1}^n \alpha_i(\bx) (W_i Y_i - \E[WY \mid \bX = \bx]),
    \end{align*}
    and then,
    \begin{align*}
        \E[\Delta_{1,n}(\bx) - \E[WY \mid \bX = \bx]] = \sum_{i=1}^n \E[\E[\alpha_i(\bx) (W_i Y_i - \E[WY \mid \bX = \bx]) \mid \bX_i]].
    \end{align*}
    Here, we use a key property of the forest growing given by Specification \ref{Spec:forests} : honesty. Indeed, it enforces that $\Dn$ is randomly split in two halves for each tree, where one part is used to build the splits, and the other half to compute the weights. Therefore, $\alpha_i(\bx, \bX_i, \bTheta_M, \Dn)$ and $W_iY_i$ are independent conditional on $\bX_i$, for all $\{i,\hdots,n\}$. Then, we have
    \begin{align*}
        \E[\Delta_{1,n}(\bx) - \E[WY \mid \bX = \bx]] =& \sum_{i=1}^n \E[\E[\alpha_i(\bx)\mid \bX_i] \E[W_i Y_i - \E[WY \mid \bX = \bx] \mid \bX_i]] \\
        =& \sum_{i=1}^n \E[\E[\alpha_i(\bx)\mid \bX_i] (\E[W_i Y_i \mid \bX_i] - \E[WY \mid \bX = \bx])].
    \end{align*}
    Since $W$ and $Y$ are independent conditional on $\bX$ from the unconfoundedness Assumption \ref{A:unconfound}, $\E[W_i Y_i \mid \bX_i] = \E[W_i \mid \bX_i] \E[Y_i \mid \bX_i]$. Additionally, Assumption \ref{A:lipschitz} states that the functions $\pi$ and $m$ are Lipschitz, and since the product of two Lipschitz functions is Lipschitz, $\E[W_i Y_i \mid \bX_i]$ is Lipschitz, with a constant $C > 0$.
    Therefore, we obtain
    \begin{align*}
        \E[\Delta_{1,n}(\bx) - \E[WY \mid \bX = \bx]] 
        \leq& \sum_{i=1}^n \E[\E[\alpha_i(\bx)\mid \bX_i] C \|\bX_i - \bx \|_2] \\
        \leq& C \E\big[ \sum_{i=1}^n \alpha_i(\bx) \|\bX_i - \bx \|_2 \big] \\
        \leq& C \E\big[ \underset{i}{\mathrm{sup}} \|\bX_i - \bx \|_2 \mathds{1}_{\alpha_i(\bx) > 0} \sum_{i=1}^n \alpha_i(\bx) \big] \\
        \leq& C \E\big[ \underset{i}{\mathrm{sup}} \|\bX_i - \bx \|_2 \mathds{1}_{\alpha_i(\bx) > 0} \big].
    \end{align*}
    Since Assumptions \ref{A:X_density} and \ref{A:lipschitz} and Specification \ref{Spec:forests} are satisfied, equation ($26$) in the Supplementary Material of \citet{athey2019generalized} states that
    \begin{align*}
        \E\big[ \underset{i}{\mathrm{sup}} \|\bX_i - \bx \|_2 \mathds{1}_{\alpha_i(\bx) > 0} \big] \longrightarrow 0,
    \end{align*}
    which gives that
    \begin{align} \label{eq:unbiased}
        \E[\Delta_{1,n}(\bx)] \longrightarrow \E[WY \mid \bX = \bx].
    \end{align}
    Next, we use equation (24) in Lemma $7$ of the Supplementary Material of \citet{athey2019generalized}, to get that $\V[\Delta_{1,n}(\bx)] = O(a_n/n)$. Since $a_n/n \longrightarrow 0$ by Specification \ref{Spec:forests}, we finally have $\V[\Delta_{1,n}(\bx)] \longrightarrow 0$.
    Finally, this last limit combined with equation (\ref{eq:unbiased}), states that $\Delta_{1,n}(\bx) - \E[WY \mid \bX = \bx]$ is asymptotically unbiased and of null variance. Using the bias-variance decomposition, we obtain the $\mathbb{L}^2$-consistency of $\Delta_{1,n}(\bx)$ towards $\E[WY \mid \bX = \bx]$, which implies the weak consistency
    \begin{align*}
        \sum_{i=1}^n \alpha_i(\bx) W_i Y_i \overset{p}{\longrightarrow} \E[WY \mid \bX = \bx].
    \end{align*}
    
    Identically, we obtain the weak consistency of the other terms involved in $\tau_{M,n}(\bx)$, i.e., $\sum_{i=1}^n \alpha_i(\bx) W_i \overset{p}{\longrightarrow} \pi(\bx)$, $\sum_{i=1}^n \alpha_i(\bx) Y_i \overset{p}{\longrightarrow} m(\bx)$, and $\sum_{i=1}^n \alpha_i(\bx) W_i^2 \overset{p}{\longrightarrow} \E[W^2 \mid \bX = \bx]$. The continuous mapping theorem gives for the last term that $\big(\sum_{i=1}^n \alpha_i(\bx) W_i\big)^2 \overset{p}{\longrightarrow} \E[W \mid \bX = \bx]^2$.
    Finally, using Slutsky's Lemma, we obtain
    \begin{align*}
        \tau_{M,n}(\bx) \overset{p}{\longrightarrow} &\frac{\E[WY \mid \bX = \bx] - \E[W \mid \bX = \bx]\E[Y \mid \bX = \bx]}{\E[W^2 \mid \bX = \bx] - \E[W \mid \bX = \bx]^2} \\
        &= \frac{\mathrm{Cov}[W, Y \mid \bX = \bx]}{\V[W \mid \bX = \bx]} \\
        &= \tau(\bx^{(\bC)}),
    \end{align*}
    where the last line is given by the local moment equation (\ref{eq:local_moment_full}), which identifies the treatment effect.
    Finally, notice that this proof applies to any linear local moment equation defining a generalized random forest.
\end{proof}

\begin{proof}[Proof of Theorem \ref{thm:cf_consistency_centered}]
    We consider $j \notin \bC$, and follow the same proof as Theorem \ref{thm:cf_consistency}, to show that the causal forest $\tau^{(-j)}_{M,n}(\bx)$ fit with $\Dn^{\star (-j)}$ converges as
    \begin{align*}
        \tau^{(-j)}_{M,n}(\bx) \overset{p}{\longrightarrow} \theta(\bx^{(-j)}),
    \end{align*}
    where $\theta(\bx^{(-j)})$ satisfies the following equation by definition of causal forests,
    \begin{align*}
        \theta(\bx^{(-j)}) \times \V[W - \pi(\bX) \mid \bX^{(-j)} = \bx^{(-j)}] - \mathrm{Cov}[W - \pi(\bX), Y - m(\bX) \mid \bX^{(-j)} = \bx^{(-j)}] = 0.
    \end{align*}
    Then, according to Proposition \ref{prop:local_moment_C}, the above moment equation identifies the treatment effect under Assumptions \ref{A:unconfound} and \ref{A:tau}, and we obtain
    \begin{align*}
        \theta(\bx^{(-j)}) = \tau(\bx^{(\bC)}),
    \end{align*}
     which gives (i). For (ii), we apply the same proof, except that the obtained local moment equation identifies $\E[\tau(\bX^{(\bC)}) \mid \bX^{(-j)} = \bx^{(-j)}]$ according to Proposition \ref{prop:local_moment_j}.
\end{proof}

\begin{proof}[Proof of Theorem \ref{thm:cf_consistency_corrected}]
    With $j \in \{1,\hdots,p\}$, recall that the causal forest $\tau_{M,n}(\bx)$ is fit with a centered dataset $\Dn^{\star}$, and the corrected causal forest estimate $\smash{\theta_{M,n}^{(-j)}(\bx)}$ is fit with $\Dn^{\star (-j)}$, an independent copy of the centered dataset with the $j$-th variable dropped, and is formally defined as
    \begin{align*}
        \theta_{M,n}^{(-j)}(\bx) = \tau_{M,n}^{(-j)}(\bx) - \frac{\sum_{i=1}^n \alpha'_i(\bx^{(-j)}) (W_i - \pi(\bX_i))^{2} \tau_{M,n}(\bX_i) - \overline{W^{2}_{\alpha'}} \overline{\tau}_{\alpha'}}{\sum_{i=1}^n \alpha'_i(\bx^{(-j)}) (W_i - \overline{W}_{\alpha'})^2},
    \end{align*}
    where $\overline{W^{2}_{\alpha'}} = \sum_{i=1}^n \alpha'_i(\bx^{(-j)}) (W_i - \pi(\bX_i))^{2}$, $\overline{\tau}_{\alpha'} = \sum_{i=1}^n \alpha'_i(\bx^{(-j)}) \tau_{M,n}(\bX_i)$, and $\overline{W}_{\alpha'} = \sum_{i=1}^n \alpha'_i(\bx^{(-j)}) (W_i - \pi(\bX_i))$.
    We first prove the convergence of the first term of the numerator,
    \begin{align*}
        \Delta_n &= \sum_{i=1}^n \alpha'_i(\bx^{(-j)}) (W_i - \pi(\bX_i))^{2} \tau_{M,n}(\bX_i) \\
        &= \sum_{i=1}^n \alpha'_i(\bx^{(-j)}) (W_i - \pi(\bX_i))^{2} \tau(\bX_i) + \sum_{i=1}^n \alpha'_i(\bx^{(-j)}) (W_i - \pi(\bX_i))^{2} (\tau_{M,n}(\bX_i) - \tau(\bX_i)).
    \end{align*}
    Using the same proof as for Theorem \ref{thm:cf_consistency}, we get that 
    \begin{align*}
        \sum_{i=1}^n \alpha'_i(\bx^{(-j)}) (W_i - \pi(\bX_i))^{2} \tau(\bX_i) \overset{p}{\longrightarrow} \E[(W - \pi(\bX))^{2} \tau(\bX) \mid \bX = \bx^{(-j)}].
    \end{align*}
    For the second term involved in $\Delta_n$, we cannot directly apply the proof of Theorem \ref{thm:cf_consistency} since the output depends on $n$ through the term $\tau_{M,n}(\bX_i)$.
    We first need to bound $\P(\alpha'_1(\bx^{(-j)}) > 0)$. Let us consider a given tree $\ell \in \{1,\hdots,M\}$, and the associated weights $\alpha'_{i \ell}(\bx^{(-j)})$ for this tree alone. From Specification \ref{Spec:truncated}, we have
    \begin{align*}
        \sum_{i=1}^n \mathds{1}_{\alpha'_{i \ell}(\bx^{(-j)}) > 0} \leq t_0,
    \end{align*}
    where $t_0$ is the maximum number of observations in each terminal leave. Since the weights are identically distributed, we have $n \E[\mathds{1}_{\alpha'_{1 \ell}(\bx^{(-j)}) > 0}] \leq t_0$, i.e., $\P(\alpha'_{1 \ell}(\bx^{(-j)}) > 0) \leq t_0 / n$. Finally, considering all trees, since $\alpha'_{1}(\bx^{(-j)}) = \sum_{\ell=1}^M \alpha'_{1 \ell}(\bx^{(-j)})/M$, we obtain
    \begin{align} \label{eq:weights}
        \P(\alpha'_{1}(\bx^{(-j)}) > 0) \leq \frac{M t_0}{n}.
    \end{align}
    Next, for the second term of $\Delta_n$, we write
    \begin{align*}
       \E[\big[\big|\sum_{i=1}^n \alpha'_i(\bx^{(-j)}) (W_i - \pi(\bX_i))^{2} (\tau_{M,n}(\bX_i) - \tau(\bX_i))\big|\big]
       \leq& \E\big[\sum_{i=1}^n \alpha'_i(\bx^{(-j)}) |\tau_{M,n}(\bX_i) - \tau(\bX_i)|\big] \\
       \leq& n \E\big[ \alpha'_1(\bx^{(-j)}) |\tau_{M,n}(\bX_1) - \tau(\bX_1)|\big].
    \end{align*}
    The right hand side of this inequality writes
    \begin{align*}
         n \E\big[ \alpha'_1(\bx^{(-j)})& |\tau_{M,n}(\bX_1) - \tau(\bX_1)|\big] \\ &=  n \E\big[ \alpha'_1(\bx^{(-j)}) |\tau_{M,n}(\bX_1) - \tau(\bX_1)| \mid  \alpha'_1(\bx^{(-j)}) > 0 \big] \P(\alpha'_{1}(\bx^{(-j)}) > 0) \\
         &\leq M t_0 \E\big[|\tau_{M,n}(\bX_1) - \tau(\bX_1)| \mid \alpha'_1(\bx^{(-j)}) > 0 \big],        
    \end{align*}
    where the last inequality is obtained using (\ref{eq:weights}). Finally, since the original causal forest trained with all inputs and the weights $\alpha'_{1}(\bx^{(-j)})$ of the retrained forest are built using independent data, the conditioning event in $\E\big[|\tau_{M,n}(\bX_1) - \tau(\bX_1)| \mid \alpha'_1(\bx^{(-j)}) > 0 \big]$ only modifies the distribution of $\bX_1$. Therefore, with $\bZ_n$ a random variable following this conditional distribution, we have
    \begin{align*}
        \E\big[|\tau_{M,n}(\bX_1) - \tau(\bX_1)| \mid \alpha'_1(\bx^{(-j)}) > 0 \big] = \E\big[|\tau_{M,n}(\bZ_n) - \tau(\bZ_n)|\big].
    \end{align*}
    Since Theorem \ref{thm:cf_consistency} gives the convergence in probability towards $0$ of $\tau_{M,n}(\bx) - \tau(\bx)$ for all $\bx \in [0,1]$ and $\bZ_n$ is independent from $\tau_{M,n}(\bx)$, we get that $\tau_{M,n}(\bZ_n) - \tau(\bZ_n) \overset{p}{\longrightarrow} 0$.
    Since the causal forest is bounded from Specification \ref{Spec:truncated}, convergence in probability implies $\mathbb{L}^1$-convergence, and we get that 
    \begin{align*}
        \E\big[|\tau_{M,n}(\bX_1) - \tau(\bX_1)| \mid \alpha'_1(\bx^{(-j)}) > 0 \big] = \E\big[|\tau_{M,n}(\bZ_n) - \tau(\bZ_n)|\big] \longrightarrow 0.
    \end{align*}
    This implies the convergence of the second term of $\Delta_n$, and overall, we obtain that
    \begin{align*}
        \Delta_n \overset{p}{\longrightarrow} \E[(W - \pi(\bX))^{2} \tau(\bX) \mid \bX = \bx^{(-j)}].
    \end{align*}
    Next, $\overline{\tau}_{\alpha'}$ is handled similarly as $\Delta_n$, and we follow the same proof as for Theorem \ref{thm:cf_consistency} to get the weak consistency of the remaining terms involved in $\theta_{M,n}^{(-j)}(\bx)$, and using Slutsky's lemma, we obtain
    \begin{align*}
        \frac{\sum_{i=1}^n \alpha'_i(\bx^{(-j)}) (W_i - \pi(\bX_i))^{2} \tau_{M,n}(\bX_i) - \overline{W^{2}_{\alpha'}} \overline{\tau}_{\alpha'}}{\sum_{i=1}^n \alpha'_i(\bx^{(-j)}) (W_i - \overline{W}_{\alpha'})^2}
        \overset{p}{\longrightarrow} \frac{\mathrm{Cov}[\tau(\bX^{(\bC)}), \pi(\bX)(1 - \pi(\bX)) \mid \bX^{(-j)} = \bx^{(-j)}]}{\V[W - \pi(\bX) \mid \bX^{(-j)} = \bx^{(-j)}]}.
    \end{align*}
    Then, following the case (ii) of Theorem \ref{thm:cf_consistency_centered}, we get
    \begin{align*}
        \tau_{M,n}^{(-j)}(\bx) \overset{p}{\longrightarrow} \frac{\mathrm{Cov}[W - \pi(\bX), Y - m(\bX) \mid \bX^{(-j)} = \bx^{(-j)}] }{\V[W - \pi(\bX) \mid \bX^{(-j)} = \bx^{(-j)}]},
    \end{align*}
    which gives the final result
    \begin{align*}
        \theta_{M,n}^{(-j)}(\bx) \overset{p}{\longrightarrow} &\frac{\mathrm{Cov}[W - \pi(\bX), Y - m(\bX) \mid \bX^{(-j)} = \bx^{(-j)}] }{\V[W - \pi(\bX) \mid \bX^{(-j)} = \bx^{(-j)}]} 
        \\ & \quad - \frac{\mathrm{Cov}[\tau(\bX^{(\bC)}), \pi(\bX)(1 - \pi(\bX)) \mid \bX^{(-j)} = \bx^{(-j)}]}{\V[W - \pi(\bX) \mid \bX^{(-j)} = \bx^{(-j)}]}
        \\ &= \E[\tau(\bX^{(\bC)}) \mid \bX^{(-j)} = \bx^{(-j)}],
    \end{align*}
    where the last equality is given by Proposition \ref{prop:local_moment_j}.
\end{proof}

\begin{proof}[Proof of Theorem \ref{thm:vimp_consistency}]
    We first consider the case $j \in \{1,\hdots,p\} \setminus \bC$ for the sake of clarity. We assume that Assumptions \ref{A:unconfound}-\ref{A:lipschitz}, and Specifications \ref{Spec:forests} and \ref{Spec:truncated} are satisfied, and causal forests are trained as specified in Theorem \ref{thm:cf_consistency_corrected}. Then, we can apply Theorems \ref{thm:cf_consistency} and \ref{thm:cf_consistency_corrected} to get that
    \begin{align*}
        \tau_{M,n}(\bX) - \theta_{M,n}^{(-j)}(\bX) \overset{p}{\longrightarrow} 0.
    \end{align*}
    According to Specification \ref{Spec:truncated}, $\tau_{M,n}(\bX) - \theta_{M,n}^{(-j)}(\bX)$ is bounded, and therefore convergence in probability implies $\mathbb{L}^2$-convergence, i.e.,
    \begin{align} \label{eq:L2_convergence}
        \E[(\tau_{M,n}(\bX) - \theta_{M,n}^{(-j)}(\bX))^2] \longrightarrow 0.
    \end{align}
    
    Next, recall that
    \begin{align*}
        \VIhat = \frac{\sum_{i = 1}^n \big[\tau_{M,n}(\bX'_i) - \theta_{M,n}^{(-j)}(\bX_i')\big]^2}{\sum_{i = 1}^n \big[\tau_{M,n}(\bX'_i) - \overline{\tau_{M,n}} \big]^2} - \mathrm{I}_n^{(0)}.
    \end{align*}
    We first consider
    \begin{align*}
        \Delta_{n,1} = \frac{1}{n} \sum_{i = 1}^n \big[\tau_{M,n}(\bX'_i) - \theta_{M,n}^{(-j)}(\bX_i')\big]^2,
    \end{align*}
    and then
    \begin{align*}
        \E[\Delta_{n,1}] = \E\big[\big(\tau_{M,n}(\bX'_1) - \theta_{M,n}^{(-j)}(\bX_1')\big)^2 \big].
    \end{align*}
    Since $|\Delta_{n,1}| = \Delta_{n,1}$, according to equation (\ref{eq:L2_convergence}), we have 
    \begin{align*}
        \E[|\Delta_{n,1}|] \longrightarrow 0,
    \end{align*}
    which also implies the convergence in probability of $\Delta_{n,1}$.
    
    Similarly for the denominator, we write
    \begin{align*}
        \Delta_{n,2} = \frac{1}{n} \sum_{i = 1}^n \tau_{M,n}(\bX'_i)^2 - \overline{\tau_{M,n}}^2 
    \end{align*}
    We first show the convergence of $\overline{\tau_{M,n}}$. Hence,
    \begin{align*}
        \E[\overline{\tau_{M,n}}] = \E[\frac{1}{n} \sum_{i = 1}^n \tau_{M,n}(\bX'_i)] = \E[\tau_{M,n}(\bX)] \longrightarrow \E[\tau(\bX^{(\bC)})],
    \end{align*}
    where the limit is obtained because Theorem \ref{thm:cf_consistency} gives the weak consistency of $\tau_{M,n}(\bX)$, which implies the convergence of the first moment since $\tau_{M,n}(\bX)$ is bounded from Specification \ref{Spec:truncated}.
    Next, we show that the variance of $\overline{\tau_{M,n}}$ vanishes. We use the law of total variance to get
    \begin{align*}
        \V[\overline{\tau_{M,n}}] = \V[\E[\overline{\tau_{M,n}} \mid \bTheta_M, \Dn]] + \E[\V[\overline{\tau_{M,n}} \mid \bTheta_M, \Dn]].
    \end{align*}
    For $\E[\V[\overline{\tau_{M,n}} \mid \bTheta_M, \Dn]]$, notice that $\tau_{M,n}(\bX'_i)$ are iid conditional on $\bTheta_M$ and $\Dn$. Therefore, 
    \begin{align*}
        \V[\overline{\tau_{M,n}} \mid \bTheta_M, \Dn] = \frac{\V[\tau_{M,n}(\bX) \mid \bTheta_M, \Dn]}{n} < \frac{K^2}{n},
    \end{align*}
    since $\tau_{M,n}(\bX)$ is bounded by $K$ from Specification \ref{Spec:truncated}. We thus obtain $\E[\V[\overline{\tau_{M,n}} \mid \bTheta_M, \Dn]] \longrightarrow 0$.
    For the first term, notice that
    \begin{align*}
        \V[\E[\overline{\tau_{M,n}} \mid \bTheta_M, \Dn]] = \V[\E[\tau_{M,n}(\bX) \mid \bTheta_M, \Dn]] < \V[\tau_{M,n}(\bX)],
    \end{align*}
    where this upper bound converges to $0$, since $\tau_{M,n}(\bX)$ converges towards $\tau(\bX^{(\bC)})$ in $\mathbb{L}^2$. 
    Overall, $\overline{\tau_{M,n}}$ is asymptotically unbiased and its variance vanishes, and therefore converges towards $0$ in $\mathbb{L}^2$, and the weak consistency follows, i.e.,
    \begin{align*}
        \overline{\tau_{M,n}} \overset{p}{\longrightarrow} \E[\tau(\bX^{(\bC)})].
    \end{align*}
    Using the continuous mapping theorem, we conduct the same analysis to get
    that $\frac{1}{n} \sum_{i = 1}^n \tau_{M,n}(\bX'_i)^2 \overset{p}{\longrightarrow} \E[\tau(\bX^{(\bC)})^2]$, and then
    \begin{align*}
        \Delta_{n,2} \overset{p}{\longrightarrow} \V[\tau(\bX^{(\bC)})],
    \end{align*}    
    with $\V[\tau(\bX^{(\bC)})] > 0$ from Assumption \ref{A:tau}.
    Finally, both the numerator $\Delta_{n,1}$ and denominator $\Delta_{n,2}$ of $\smash{\VIhat}$ converge in probability, and we can apply Slutsky's Lemma to obtain
    \begin{align*}
        \VIhat + \mathrm{I}_n^{(0)} \overset{p}{\longrightarrow} 0,
    \end{align*}
    and following the same arguments, we get that $\mathrm{I}_n^{(0)} \overset{p}{\longrightarrow} 0$, which gives the final result.
    The proof is similar for the case where $j \notin \bC$.
\end{proof}

\begin{proof}[Proof of Theorem \ref{thm:vimp_consistency_bias}]
    We can directly deduce from the proof of Theorem \ref{thm:cf_consistency_corrected} that, for $\bx \in (0,1)$,
    \begin{align*}
        \tau_{M,n}^{(-j)}(\bx) \overset{p}{\longrightarrow}
        \E[\tau(\bX^{(\bC)}) \mid \bX^{(-j)} = \bx^{(-j)}] + \frac{\mathrm{Cov}[\tau(\bX^{(\bC)}), \pi(\bX)(1 - \pi(\bX)) \mid \bX^{(-j)} = \bx^{(-j)}]}{\V[W - \pi(\bX) \mid \bX^{(-j)} = \bx^{(-j)}]}.
    \end{align*}
    We denote by $C_j(\bx^{(-j)})$ the second term of the above limit to lighten notations.
    Next, we follow the proof of Theorem \ref{thm:vimp_consistency} to get the convergence of $\mathcal{I}_n^{(j)}$, given by
    \begin{align*}
        \mathcal{I}_n^{(j)} \overset{p}{\longrightarrow} 
        \frac{\E[(\tau(\bX^{(\bC)}) - \E[\tau(\bX^{(\bC)}) \mid \bX^{(-j)}] - C_j(\bX^{(-j)}))^2]}{\V[\tau(\bX^{(\bC)}]}.
    \end{align*}
    The numerator writes
    \begin{align*}
        \E[(\tau(\bX^{(\bC)}) &- \E[\tau(\bX^{(\bC)}) \mid \bX^{(-j)}] - C_j(\bX^{(-j)}))^2] \\
        = &\E[\E[(\tau(\bX^{(\bC)}) - \E[\tau(\bX^{(\bC)}) \mid \bX^{(-j)}] - C_j(\bX^{(-j)}))^2 \mid \bX^{(-j)}]] \\
        = &\E[(\tau(\bX^{(\bC)}) - \E[\tau(\bX^{(\bC)}) \mid \bX^{(-j)}])^2 + C_j(\bX^{(-j)})^2] \\ &- 2 \E[ \E[\tau(\bX^{(\bC)}) - \E[\tau(\bX^{(\bC)}) \mid \bX^{(-j)}] \mid \bX^{(-j)}] \E[C_j(\bX^{(-j)}))^2 \mid \bX^{(-j)}] ] \\
        = &\E[(\tau(\bX^{(\bC)}) - \E[\tau(\bX^{(\bC)}) \mid \bX^{(-j)}])^2] + \E[C_j(\bX^{(-j)})^2].
    \end{align*}
    Then, we have
    \begin{align*}
        \mathcal{I}_n^{(j)} \overset{p}{\longrightarrow} \mathrm{I}^{(j)} + 
        \frac{\E[C_j(\bX^{(-j)})^2]}{\V[\tau(\bX^{(\bC)}]},
    \end{align*}
    which gives the final result.
\end{proof}

\end{document}